\documentclass{article}

\usepackage[preprint]{style/neurips_2026}

\usepackage[utf8]{inputenc}
\usepackage[T1]{fontenc}
\usepackage{hyperref}
\usepackage{url}
\usepackage{pifont}
\usepackage{booktabs}
\usepackage{amsfonts}
\usepackage{amsmath}
\usepackage{amssymb}
\usepackage{amsthm}
\usepackage{nicefrac}
\usepackage{microtype}
\usepackage{xcolor}
\usepackage{graphicx}
\usepackage{wrapfig}
\usepackage{multirow}
\usepackage{rotating} 
\usepackage{array}  
\usepackage{algorithm}
\usepackage{caption}
\usepackage{algpseudocode}
\usepackage{wasysym}
\usepackage{float}

\newtheorem{lemma}{Lemma}
\title{CM-EVS: Sparse Panoramic RGB-D-Pose Data \\for Complete Scene Coverage}

\author{%
  Jiale Liu$^{*1}$ \quad
  Jungang Li$^{*2}$ \quad
  Jieming Yu$^{*3}$ \quad
  Xinglin Yu$^{*5}$ \\
  Zihao Dongfang$^{2}$ \quad
  Keyu Shi$^{7}$ \quad
  Zongjian Ding$^{4}$ \quad
  Jiahuan Zhang$^{6}$ \\
  Shunwen Bai$^{1}$ \quad
  Haoran Huang$^{1}$ \quad
  Yurun Wang$^{1}$ \\
  Yanxi Wu$^{1}$ \quad
  Ningzhe Yu$^{1}$ \quad
  Yudong Gao$^{\dagger 3}$ \quad
  Mingjun Cheng$^{\dagger 1}$ \\[6pt]
  $^{1}$Zhejiang University \quad
  $^{2}$The Hong Kong University of Science and Technology (Guangzhou) \\
  $^{3}$The Hong Kong University of Science and Technology \quad
  $^{4}$University of Chinese Academy of Sciences \\
  $^{5}$Xinjiang University \quad
  $^{6}$Tianjin University \quad
  $^{7}$Harbin Institute of Technology \\[4pt]
  $^{*}$Equal contribution. \quad $^{\dagger}$Corresponding authors. \\[2pt]
  \texttt{ygaodj@connect.ust.hk} \quad \texttt{mkellerc@outlook.com}
}

\hypersetup{
  pdftitle={CM-EVS: Sparse Panoramic RGB-D-Pose Data for Complete Scene Coverage},
  pdfauthor={Jiale Liu, Jungang Li, Jieming Yu, Xinglin Yu, Zihao Dongfang, Zongjian Ding, Kaifeng Ding, Yi Yang, Lidong Chen, Yang Zou, Shunwen Bai, Jiahuan Zhang, Haoran Huang, Shan Huang, Yudong Gao, Mingjun Cheng},
  pdfkeywords={panoramic RGB-D dataset, equirectangular projection, viewpoint selection, dataset curation, submodular optimization, novel view synthesis}
}

\begin{document}

\maketitle

\begin{abstract}
Modern 3D visual learning relies on observations sampled from metric 3D assets, yet existing scans, meshes, point clouds, simulations, and reconstructions do not directly provide a sparse, comparable, and geometry-consistent panoramic training interface. Dense trajectories duplicate nearby views, source-specific rendering policies yield heterogeneous annotations, and sparse heuristics may miss important regions or introduce depth-inconsistent observations. \emph{We study how to convert 3D assets into sparse panoramic RGB-D-pose data that preserves complete scene coverage with low redundancy and auditable provenance.} We propose \textbf{COVER} (\textbf{\underline{C}}overage-\textbf{\underline{O}}riented \textbf{\underline{V}}iewpoint curation with \textbf{\underline{E}}RP \textbf{\underline{R}}ange-depth warping), a training-free ERP viewpoint curator that projects geometry observed from selected views into candidate ERP probes, scores incremental coverage, and penalizes depth conflicts. Under bounded proxy error, its greedy coverage proxy preserves the standard coverage-style approximation behavior up to an additive error term. Using COVER, we build \textbf{CM-EVS} (\textbf{\underline{C}}overage-curated \textbf{\underline{M}}etric \textbf{\underline{E}}RP \textbf{\underline{V}}iew \textbf{\underline{S}}et), a panoramic RGB-D-pose dataset with 36{,}373 curated ERP frames from 1{,}275 indoor scenes across Blender indoor, HM3D, and ScanNet++, complemented by outdoor panoramas from TartanGround and OB3D re-encoded into the same schema. Each frame provides full-sphere RGB, metric range depth, calibrated pose; COVER-produced indoor frames include per-step provenance logs. With a median of only 25 frames per indoor scene, CM-EVS covers all 13 unified room types while maintaining compact scene-level coverage. Experiments show that COVER improves the coverage--conflict trade-off, making CM-EVS a sparse, compact, and auditable RGB-D-pose resource for geometry-consistent panoramic 3D learning.
\end{abstract}

\section{Introduction}
\label{sec:intro}

\begin{figure}[!t]
\centering
\includegraphics[width=\linewidth,height=6cm]{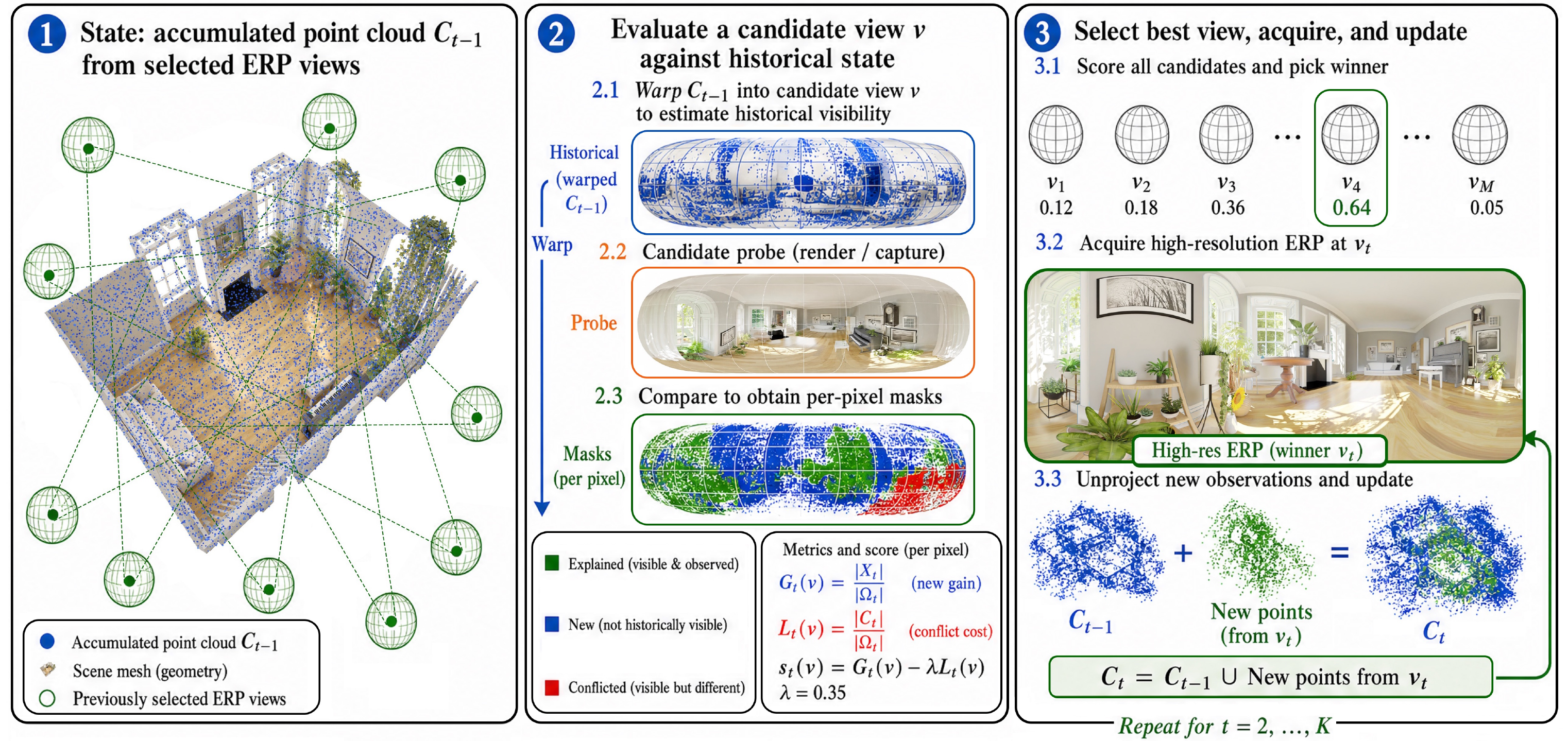}
\vspace{-4mm}
\caption{Overview of CM-EVS. An expanded illustration of the \textsc{COVER} pipeline (warping oracle, score, and point-cloud update) is in Appendix~\ref{app:overview_expanded} (Figure~\ref{fig:teaser_expanded}).}
\label{fig:teaser}
\vspace{-8mm}
\end{figure}

Modern 3D visual learning relies on observations sampled from metric 3D assets, including scans, meshes, point clouds, simulated environments, and reconstructed scenes. Among different observation formats, panoramic RGB-D-pose data offers a compact interface between scene-scale geometry and model training, as it converts scene-scale 3D structure into dense, view-centered supervision while preserving global spatial context: a single equirectangular projection (ERP) frame records a full $4\pi$ solid angle from one camera center, follows a shared spherical ray parameterization, and aligns appearance, metric range depth, and calibrated pose in a unified representation~\citep{zheng2025panorama}. This makes ERP observations useful for panoramic depth estimation~\citep{shen2022panoformer}, panoramic NeRF and Gaussian Splatting reconstruction~\citep{wang2024perf}, and 360$^\circ$ scene generation~\citep{wang2024360dvd}. However, 3D assets do not by themselves define an effective panoramic training interface. Models learn from sampled observations, and the sampling policy determines their coverage, redundancy, geometric consistency, and reproducibility.

This paper studies the \emph{observation layer} between metric 3D assets and panoramic model training: \textit{how to select and standardize panoramic RGB-D-pose views that are compact, geometrically informative, and auditable.} The challenge is not simply to render more ERP frames, but to expose non-redundant scene geometry while avoiding depth-inconsistent observations. Dense trajectories repeatedly sample nearby viewpoints, sparse heuristics may miss important regions, and source-specific rendering policies make datasets difficult to compare, since equal frame counts can encode very different geometric evidence. Existing resources reflect these limitations from different angles: captured or per-paper panoramas~\citep{albanis2021pano3d,bertel2020omniphotos} are often tied to fixed protocols or limited budgets; trajectory-based corpora such as 360DVD~\citep{wang2024360dvd} and Matrix-3D~\citep{zhang2025matrix3d} prioritize video continuity or generation rather than marginal coverage; and large 3D asset datasets such as Hypersim~\citep{roberts2021hypersim}, Structured3D~\citep{zheng2020structured3d}, HM3D, and ScanNet++~\citep{yeshwanth2023scannetpp} provide rich geometry but leave panoramic view generation to source-specific or downstream sampling choices. Moreover, candidate viewpoints, coverage gains, conflict statistics, and selection scores are rarely released as first-class artifacts, making panoramic observation sets hard to reproduce, diagnose, or extend.

We address this gap with \textbf{COVER} (\textbf{\underline{C}}overage-\textbf{\underline{O}}riented \textbf{\underline{V}}iewpoint curation with \textbf{\underline{E}}RP \textbf{\underline{R}}ange-depth warping), a training-free ERP viewpoint curator that formulates panoramic view selection as conflict-aware coverage maximization. Given a candidate ERP pool, COVER accumulates selected range-depth observations into a point cloud, projects the accumulated geometry into low-resolution probes of remaining candidates, and greedily selects views that reveal uncovered regions while penalizing range-depth conflicts with already observed geometry (Figure~\ref{fig:teaser}). This gives a compact, reproducible, training-free policy with a bounded-error analysis of the greedy coverage proxy.

We use COVER to build \textbf{CM-EVS} (\textbf{\underline{C}}overage-curated \textbf{\underline{M}}etric \textbf{\underline{E}}RP \textbf{\underline{V}}iew \textbf{\underline{S}}et), a provenance-tracked panoramic RGB-D-pose dataset for sparse yet complete scene coverage. Its curated indoor core contains 36{,}373 ERP frames from 1{,}275 scenes across Blender indoor, HM3D, and ScanNet++, complemented by schema-compatible outdoor panoramas re-encoded from TartanGround and OB3D. Each sample provides full-sphere RGB, metric range depth along ERP rays, and calibrated pose; COVER-produced frames further include candidate pools, coverage gains $G_t$, depth-conflict ratios $L_t$, and selection scores $s_t$. With a median of only $\sim$25 ERP frames per indoor scene, CM-EVS covers all 13 unified room types, and COVER improves the coverage--conflict trade-off over random, single-view-probe, coverage-only, and low-conflict-only baselines. CM-EVS thus offers a sparse, compact, and auditable panoramic RGB-D-pose resource for 3D learning.

Our contributions are summarized as follows.
\par\noindent\ding{182} \emph{We propose \textbf{COVER}, a conflict-aware ERP viewpoint curator.} COVER is a training-free greedy selector that uses coverage-oriented range-depth warping to choose high-coverage, low-conflict panoramic RGB-D-pose views, with a bounded-error analysis of its coverage proxy.
\par\noindent\ding{183} \emph{We introduce \textbf{CM-EVS}, a compact and provenance-tracked panoramic RGB-D-pose corpus.} CM-EVS contains a COVER-curated indoor core of 36{,}373 ERP frames from 1{,}275 scenes, complemented by a schema-compatible outdoor extension, with full-sphere RGB, metric range depth, calibrated pose, unified room labels, and per-frame provenance logs.
\par\noindent\ding{184} \emph{We evaluate auditable panoramic observation efficiency.} We release candidate pools, coverage gains, depth-conflict ratios, and selection scores, and show that COVER improves the coverage--conflict trade-off over random, single-view-probe, coverage-only, and low-conflict-only baselines.

By making panoramic data construction compact, geometry-aware, and reproducible, CM-EVS offers an auditable observation layer for evaluating and training geometry-consistent panoramic 3D models.
\section{Related Work}
\label{sec:rw}
\textbf{Panoramic Data for 3D Learning.}
Panoramic RGB-D-pose observations provide a compact interface for 3D perception, reconstruction, and generation, since a single ERP frame captures a full $4\pi$ field of view under a unified spherical parameterization. Existing 3D scene resources~\citep{chang2017matterport3d,yeshwanth2023scannetpp,ramakrishnan2021hm3d,roberts2021hypersim,zheng2020structured3d,patel2025tartanground,ito2025ob3d} provide rich geometry, annotations, or simulation environments, and panoramic datasets and reconstruction / generation methods~\citep{albanis2021pano3d,wang2024360dvd,zhang2025matrix3d,ou2026holo360d,wang2024perf,chen2023panogrf,zhou2024dreamscene360,tang2023mvdiffusion} further highlight the value of full-sphere observations. However, these resources typically inherit source-specific capture protocols, dense trajectories, or per-paper view-construction pipelines, so equal frame counts can encode substantially different geometric evidence and the camera policy behind a dataset is rarely released as a reproducible artifact. \emph{CM-EVS instead targets the data-supply layer: it converts heterogeneous 3D assets into sparse, calibrated, and comparable panoramic RGB-D-pose observations, making the observation policy behind panoramic 3D learning explicit and auditable.}

\textbf{View Selection for Data Curation.} 
View planning and next-best-view methods~\citep{vasquez2014volumetric,pan2022scvp,pan2022activenerf,ran2023neurar,chen2024gennbv} instead study online camera-pose selection for active reconstruction or exploration. \textsc{COVER} sits in a complementary regime: an offline, training-free, fixed-budget curator that builds panoramic training data from existing 3D assets by balancing incremental coverage with depth-conflict penalties. \emph{CM-EVS releases candidate pools, coverage gains, conflict ratios, selection scores, and provenance logs, following \emph{Datasheets for Datasets}~\citep{gebru2021datasheets} and Croissant~\citep{mlcommons2024croissant}, so users can reproduce the view policy, diagnose failure cases, or replace \textsc{COVER} with alternative strategies under the same candidate space.} Per-area discussion and additional citations are in Appendix~\ref{app:related-work}.

\section{Method}
\label{sec:method}

To select panoramic RGB-D-pose views that are compact, geometrically informative, and auditable, we propose \textsc{COVER}, a training-free ERP viewpoint curator that casts panoramic view selection as conflict-aware coverage maximization. We formalize fixed-budget viewpoint selection and define \textsc{COVER}'s conflict-aware warping oracle (\S\ref{sec:warping}), state the approximation guarantee and package the algorithm (\S\ref{sec:theory}), and describe the per-scene pipeline and per-source adapters (\S\ref{sec:pipeline}).

\subsection{Problem setup}
\label{sec:input}

Let $\mathcal{S}$ be a 3D scene (mesh, point cloud, or renderer-native asset) with a finite candidate set $\mathcal{P}\subset\mathbb{R}^3$ proposed by a source-specific adapter (\S\ref{sec:pipeline}). A geometric-validity predicate $\varphi(v,\mathcal{S})\!\in\!\{0,1\}$ rejects candidates embedded in geometry, flush against a wall, occluded by clutter, or otherwise physically implausible (Appendix~\ref{app:filter}); the feasible set is $\mathcal{P}_\varphi=\{v\in\mathcal{P}:\varphi(v,\mathcal{S})=1\}$. Discretize the observable surface of $\mathcal{S}$ into elements $\Omega(\mathcal{S})$ and let $O(v;\mathcal{S})\!\subseteq\!\Omega(\mathcal{S})$ be those observed from $v$. Given budget $K$, \textsc{COVER} solves the fixed-budget coverage problem
\begin{equation}
  \max_{\mathcal{V}\subseteq\mathcal{P}_\varphi,\;|\mathcal{V}|\le K}\;
  \Big|\bigcup_{v\in\mathcal{V}} O(v;\mathcal{S})\Big|,
\end{equation}
returning $\mathcal{V}$ together with per-frame ERP RGB, range depth, and pose. This is Max-$k$-Cover (NP-hard; no $(1-1/e+\epsilon)$-approximation unless $\mathrm{P}=\mathrm{NP}$ \citep{karp1972reducibility,feige1998threshold}); greedy with exact marginal gains achieves the $(1-1/e)$ bound \citep{nemhauser1978analysis}. \textsc{COVER} solves this greedily, with $\mathcal{V}_{t-1}$ the partial selection at step $t$ and $\mathcal{C}_{t-1}$ the point cloud unprojected from its range depth.

\subsection{Conflict-aware warping oracle}
\label{sec:warping}

\paragraph{Why warping.} An exact greedy oracle would render every $v\in\mathcal{P}_\varphi$ at full resolution per step ($10^2$--$10^3\!\times$ the cost of the final $K$ frames). \textsc{COVER} instead scores candidates with a cheap warping proxy and renders only the winner at full resolution. The resulting per-step proxy error $\epsilon_t$ is absorbed by an additive penalty in our coverage guarantee (Lemma~\ref{lem:noisy}, \S\ref{sec:theory}).

\paragraph{Oracle.} At step $t$, to score a candidate $v$ given the partial state $(\mathcal{V}_{t-1}, \mathcal{C}_{t-1})$, we run two cheap low-resolution passes: warping renders $\mathcal{C}_{t-1}$ into $v$'s ERP frame, marking pixels $H_v$ already explained by history (with predicted depth $D_v^{\text{hist}}$); probing renders $v$ itself, marking pixels $Q_v$ visible from $v$ (with probe depth $D_v^{\text{probe}}$). With depth tolerance $\delta\!=\!0.5\%$ of the AABB diagonal (clamped per source, Appendix~\ref{app:filter}), probe pixels split into agreed, new, and conflicting:
\begin{equation}
\begin{aligned}
E_v &= Q_v \cap H_v \cap \{|D_v^{\text{probe}} - D_v^{\text{hist}}| \le \delta\} && \text{(explained),}\\
N_v &= Q_v \setminus H_v && \text{(new),}\\
C_v &= Q_v \cap H_v \cap \{|D_v^{\text{probe}} - D_v^{\text{hist}}| > \delta\} && \text{(conflicted).}
\end{aligned}
\end{equation}
Normalizing by the total probe-pixel count $|\Omega_v|$ gives a coverage gain and a conflict penalty,
\begin{equation}
G_t(v)=\frac{|N_v|}{|\Omega_v|},\qquad
L_t(v)=\frac{|C_v|}{|\Omega_v|},\qquad
s_t(v)=G_t(v)-\lambda L_t(v).
\end{equation}
Because $N_v$ and $C_v$ are disjoint, $\lambda$ re-ranks candidates rather than rescaling them. We use $\lambda\!=\!0.35$ throughout and ablate the choice in \S\ref{sec:lambda}.
\subsection{Theoretical guarantee and algorithm}
\label{sec:theory}

Standard noisy-oracle analysis of greedy submodular maximization \citep{krause2014submodular,hassidim2017submodular,badanidiyuru2014streaming,mirzasoleiman2018streaming} guarantees $f(\mathcal{V}_K)\ge(1-1/e)f(\mathcal{V}^*)-2\sum_t\epsilon_t$ under bounded per-step proxy error $\epsilon_t$. Allowing the depth-conflict ratio to amplify proxy uncertainty yields:

\begin{lemma}[Conflict-aware noisy oracle]
\label{lem:noisy}
Let $\Delta_t(v)=\Delta f(v\mid\mathcal{V}_{t-1})$ be the true marginal coverage and $\widehat\Delta_t(v)=G_t(v)$ the warping-oracle proxy. Suppose there is $\eta\ge0$ such that $|\widehat\Delta_t(v)-\Delta_t(v)|\le\epsilon_t+\eta L_t(v)$ for every candidate. Run conflict-aware greedy with $s_t(v)=G_t(v)-\lambda L_t(v)$ and $\lambda\ge\eta$, and let $\gamma_t=L_t(v_t^*)$ for an oracle-best candidate $v_t^*$. Then
\begin{equation}
\label{eq:noisy_oracle_bound}
f(\mathcal{V}_K)
\;\ge\;
(1-1/e)f(\mathcal{V}^*)
-
\sum_{t=1}^{K}(2\epsilon_t+2\lambda\gamma_t).
\end{equation}
\end{lemma}

The proof is in Appendix~\ref{app:proof}. The constant $\eta$ is not assumed known a priori: the conflict weight $\lambda\!=\!0.35$ used in the rest of the paper is validated by the $\lambda$ sensitivity sweep (\S\ref{sec:lambda}), which shows a wide stable plateau in $\lambda\!\in\![0.1,0.5]$ that absorbs reasonable $\eta$ mis-estimation.

\textbf{Algorithm.} Algorithm~\ref{alg:greedy} packages the conflict-aware greedy loop. Starting from a seed $v_0$ chosen from interior candidates, \textsc{COVER} iterates $K\!-\!1$ rounds: warp the accumulated cloud into each remaining candidate, score by the conflict-aware $s_t$, render the chosen candidate, and update the cloud. The seed is shared across all baselines in \S\ref{sec:cov}, so coverage gains are not inflated by seed choice. Hyperparameter defaults and the production-side adaptive frame-budget heuristic (gain-gradient early stop) are deferred to Appendix~\ref{app:hyperparams}.
\begin{algorithm}[h]
\caption{\textsc{COVER}: Conflict-Aware Budgeted Greedy ERP View Selection.}
\label{alg:greedy}
\begin{algorithmic}[1]
\Require feasible candidates $\mathcal{P}_\varphi$, scene $\mathcal{S}$, budget $K$, conflict weight $\lambda$, probe resolution $h\!\times\!w$, seed pool size $M_0$
\Ensure selected viewpoints $\mathcal{V}$ and per-frame ERP RGB / depth / pose
\State $v_0 \gets \arg\max$ single-view probe coverage among the $M_0$ feasible candidates closest to the AABB center  \Comment{seed pool restricts $v_0$ to interior, not boundary, candidates}
\State render $v_0$; $\mathcal{C} \gets$ unproject range\_depth$(v_0)$
\State $\mathcal{V} \gets \{v_0\}$
\While{$|\mathcal{V}| < K$}
  \ForAll{$v \in \mathcal{P}_\varphi \setminus \mathcal{V}$}
    \State $(H_v, D_v^{\text{hist}}) \gets$ warp$(\mathcal{C}, v)$
    \State $(Q_v, D_v^{\text{probe}}) \gets$ low-res probe at $v$
    \State compute $(E_v, N_v, C_v)$, $G_t(v)$, $L_t(v)$, $s_t(v)$ \Comment{\S\ref{sec:warping}}
  \EndFor
  \State $v_t \gets \arg\max_v s_t(v)$ \Comment{reachability used only as tie-breaker}
  \State render $v_t$; $\mathcal{C} \gets \mathcal{C} \cup$ unproject range\_depth$(v_t)$
  \State $\mathcal{V} \gets \mathcal{V} \cup \{v_t\}$
\EndWhile
\State \Return $\mathcal{V}$ and per-frame (RGB, range depth, pose)
\end{algorithmic}
\end{algorithm}

\subsection{Pipeline}
\label{sec:pipeline}
\label{sec:adapters}

The release ships two adapter classes. \textbf{Curator adapters} (Blender indoor, HM3D, ScanNet++) plug a source into the three-phase pipeline below. \textbf{Re-encoding adapters} (TartanGround \citep{patel2025tartanground}, OB3D \citep{ito2025ob3d}) take sources that already provide dense RGB-D-pose trajectories and convert them into the unified ERP + pose schema (\S\ref{sec:schema}) without running \textsc{COVER}: outdoor frames are full re-encoded source trajectories, not curator-selected subsets, so they do not carry the per-step provenance log. Per-source detail is in Table~\ref{tab:3_adapters} (Appendix~\ref{app:adapters}); failure modes are catalogued in Appendix~\ref{app:failures}.

\begin{figure}[htbp]
\centering
\includegraphics[width=0.89\linewidth]{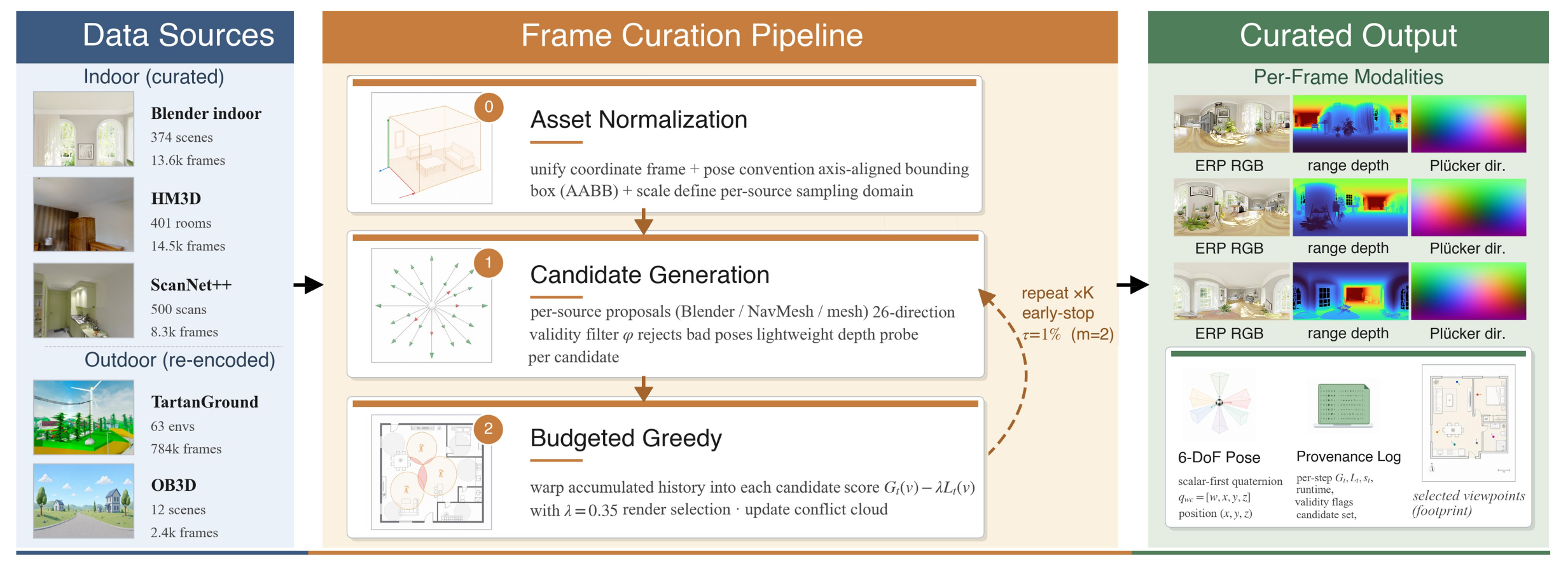}  
\caption{\textsc{COVER}'s three-phase per-scene pipeline (Algorithm~\ref{alg:greedy}). Each iteration emits one ERP RGB-depth-pose frame plus its per-step provenance log.}
\label{fig:3_pipeline}
\vspace{-4mm}
\end{figure}


\textsc{COVER} runs three phases per scene (Figure~\ref{fig:3_pipeline}), driven by a per-source adapter (handles Phases 0--1).

\textbf{Phase 0 (asset normalization).} The adapter loads the source, converts coordinates and pose into the unified schema (specified in \S\ref{sec:schema}), and computes the AABB.

\textbf{Phase 1 (candidate generation).} Candidates are proposed in a source-specific way (grid + height layers for Blender indoor, rendered with a procedural pipeline in the spirit of BlenderProc \citep{denninger2019blenderproc}; NavMesh / label-based room proposals for HM3D, derived from Habitat-Sim \citep{savva2019habitat}; mesh / point-cloud proposals for ScanNet++) and filtered by the 26-direction validity predicate $\varphi$ (Appendix~\ref{app:filter}); these thresholds are reported for auditability, not learned.

\textbf{Phase 2 (budgeted greedy).} Starting from a common seed $v_0$, the warping oracle scores remaining candidates, the chosen candidate is rendered at high resolution, and the accumulated point cloud is updated, repeating for $K\!-\!1$ rounds (Algorithm~\ref{alg:greedy}).

\section{The CM-EVS Dataset}
\label{sec:data}

We apply \textsc{COVER} across Blender indoor, HM3D, and ScanNet++ to build CM-EVS, a provenance-tracked panoramic RGB-D-pose dataset, complemented by schema-compatible outdoor panoramas re-encoded from TartanGround and OB3D. We specify the release's schema, composition, and cross-dataset position (\S\ref{sec:schema}), then characterize the four properties that distinguish CM-EVS (\S\ref{sec:datasetanalysis}).

\subsection{Release specifications}
\label{sec:schema}
\label{sec:composition}

\textbf{Schema and pose convention.} The world frame is right-handed with $+X$ right, $+Y$ up, $+Z$ forward; the camera frame follows OpenCV ($+x$ right, $+y$ down, $+z$ forward). Extrinsics are a scalar-first world-to-camera quaternion $q_{wc}=[w,x,y,z]$ and a position $C_w$ expressed relative to the scene's first selected frame, so a world point $p_w$ projects as $p_c=R_{wc}(p_w-C_w)$. ERP pixels use the standard spherical-CNN convention, longitude $(u/W-0.5)\,2\pi$ and latitude $(0.5-v/H)\,\pi$. Each frame ships RGB ($2048\!\times\!1024$ for Blender indoor, source-native otherwise), float32 range depth in metres, and pose; \textsc{COVER}-produced scenes additionally carry per-step logs of $(G_t, L_t, s_t)$ with the selected and candidate viewpoints. Scene-level splits keep frames from the same scene or space unit together.

\textbf{Composition.} Table~\ref{tab:4_composition} reports the per-source distribution; per-scene frame counts are not fixed but follow the gain-gradient early stop (Appendix~\ref{app:hyperparams}), with the resulting distributions characterized in \S\ref{sec:datasetanalysis}(d). Resolution differs across sources because real-scan inputs (HM3D, ScanNet++) carry source-side geometric and texture limits below $2048\!\times\!1024$; we render or re-encode at native source resolution rather than upsample.

\begin{table}[h]
\caption{Per-source dataset composition. Frames-per-scene reports median (IQR) over all scenes in the source. ``not redist.''~$=$~code MIT, frames not redistributed. }
\label{tab:4_composition}
\centering
\small
\setlength{\tabcolsep}{4pt}
\resizebox{\linewidth}{!}{%
\begin{tabular}{lrrlllr}
\toprule
Source & Scenes & Frames & Frames/scene\,(med, IQR) & Resolution & Median depth & License \\
\midrule
Blender indoor          & 374   & 13{,}631  & 33\,(27, 53)             & 2048$\times$1024 & 1.85\,m & CC-BY 4.0 \\
Outdoor / TartanGround  & 63 (envs)& 783{,}944& 9{,}360\,(4{,}915, 24{,}601) & 2048$\times$1024 & 3.63\,m\textsuperscript{$\dagger$} & per source \\
Outdoor / OB3D          & 12    & 2{,}400   & 200\,(200, 200)          & 1600$\times$800  & 3.88\,m & per source \\
HM3D                    & 401   & 14{,}475  & 37\,(26, 48)             & 1024$\times$512  & 1.63\,m & not redist. \\
ScanNet++               & 500   & 8{,}267   & 16\,(12, 20)             & 1024$\times$512  & 1.38\,m & not redist. \\
\midrule
Total & 1{,}350& 822{,}717& --- & mixed & --- & mixed \\
\bottomrule
\end{tabular}}
\vspace{-4mm}
\end{table}


\begin{table}[h]
\caption{CM-EVS vs.\ existing panoramic / 3D-scene resources. \CIRCLE\,= yes, \Circle\,= no, \LEFTcircle\,= partial. ``Frames/scene'' is the median per-scene frame count under each corpus's release policy.}
\label{tab:4_compare}
\centering
\small
\setlength{\tabcolsep}{4pt}
\resizebox{\linewidth}{!}{%
\begin{tabular}{lrrrcccll}
\toprule
Resource & \# scenes & \# frames & Frames/scene & Resolution & RGB-D & ERP & Trajectory & Release \\
\midrule
Matterport3D                  & 90 bldgs   & 10.8K / 194K       & $\sim$120 & src-dep.        & \CIRCLE & \LEFTcircle & operator & gated \\
ScanNet                       & 1{,}513    & 2.5M               & $\sim$1{,}650 & 640$\times$480  & \CIRCLE & \Circle & handheld & gated \\
ScanNet++                     & 1{,}006    & 280K + 3.7M        & $\sim$3{,}960 & src-dep.        & \CIRCLE & \Circle & handheld & gated \\
HM3D                          & 1{,}000    & user-rendered      & user-set & sim-dep.        & \LEFTcircle & \Circle & simulator & gated \\
Hypersim                      & 461        & 77K                & 168 & 1024$\times$768 & \CIRCLE & \Circle & toolkit & public \\
Pano3D-style                  & $\sim$100  & $\sim$3.6K         & $\sim$36 & 2048$\times$1024 & \CIRCLE & \CIRCLE & derived & mixed \\
360DVD                        & 2{,}114    & 211K               & 100 & 1280$\times$720 & \Circle & \CIRCLE & scripted video & mixed \\
Matrix-Pano                   & 116.8K     & 16M                & $\sim$138 & src-dep.        & \CIRCLE & \CIRCLE & generative & mixed \\
\textbf{CM-EVS (curator)}     & 1{,}275    & \textbf{36{,}373}  & \textbf{$\sim$25 (median)} & native       & \CIRCLE & \CIRCLE & \textbf{CM-EVS} & mixed \\
\textbf{CM-EVS Blender indoor} & 374        & \textbf{13{,}631}  & 33 (med, IQR 27--53) & 2048$\times$1024 & \CIRCLE & \CIRCLE & \textbf{CM-EVS} & \textbf{CC-BY 4.0} \\
\textbf{CM-EVS total}    & 1{,}350    & 822{,}717          & --- & mixed            & \CIRCLE & \CIRCLE & \textbf{CM-EVS} & mixed \\
\bottomrule
\end{tabular}}
\end{table}

\subsection{Distinguishing properties}
\label{sec:datasetanalysis}
\label{sec:atglance}

We characterize CM-EVS along four distinguishing properties (Figure~\ref{fig:4_simple_overview_of_CM-EVS}): (a) \emph{multi-view $4\pi$ coverage}, (b) \emph{unified RGB-D-pose schema}, (c) \emph{scene-type diversity}, and (d) \emph{low redundancy at scale}. Per-frame quality statistics and the 50-frame audit are deferred to Appendix~\ref{app:audit}.
\begin{figure}[h]
\centering
\includegraphics[width=\linewidth]{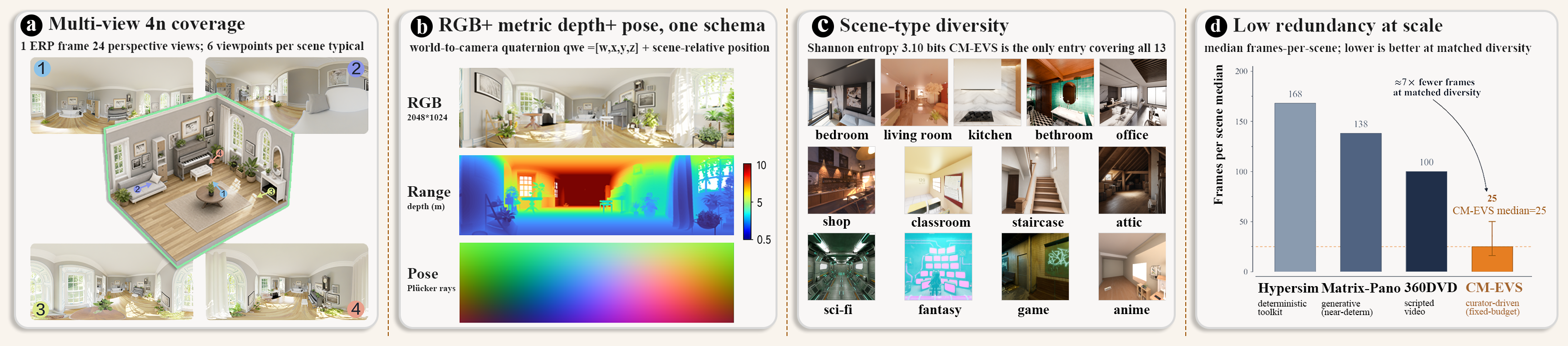}
\caption{CM-EVS: (a) multi-view $4\pi$ coverage, (b) RGB + metric range depth + pose under one schema, (c) scene-type diversity across 13 unified buckets, and (d) low redundancy at scale.}
\label{fig:4_simple_overview_of_CM-EVS}
\end{figure}

\textbf{(a) Multi-view $4\pi$ coverage.} Each scene's selected ERP viewpoints form a multi-view set spanning the space, with every viewpoint contributing a full $4\pi$ sphere rather than a slice (Figure~\ref{fig:4_simple_overview_of_CM-EVS}); a detailed example on a Blender indoor residential scene with six \textsc{COVER}-selected viewpoints spanning three functional zones (entryway, living area, bedroom alcove) and the accumulated point-cloud overlay is in Appendix~\ref{app:audit} (Figure~\ref{fig:case_multiview}).

\begin{figure}[h]
\centering
\includegraphics[width=0.99\linewidth]{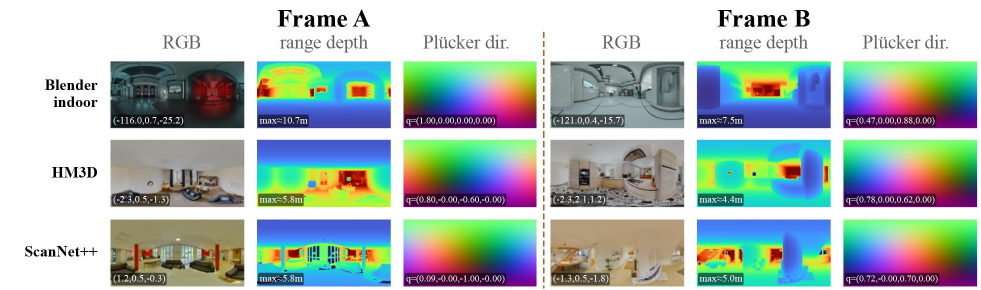}
\caption{Representative ERP frames per source. Each frame ships three modalities: RGB, range depth (Turbo colormap), and per-pixel Pl\"ucker ray direction \citep{levoy1996lightfield}.}
\label{fig:4_sample_mosaic}
\vspace{-4mm}
\end{figure}

\textbf{(b) Unified RGB-D-pose schema.} Every frame ships RGB, ERP range depth, and pose (\S\ref{sec:schema}); Figure~\ref{fig:4_sample_mosaic} shows the three modalities co-rendered per source. Per-source depth distributions span $0.3$--$30+$\,m for Blender indoor and concentrate around $1.4$--$1.9$\,m for HM3D and ScanNet++, with outdoor sources extending to tens of metres (Appendix~\ref{app:audit}).

\textbf{(c) Scene-type diversity.} We bucket scenes into 13 coarse room-type categories (Appendix~\ref{app:audit}). Figure~\ref{fig:4_xdataset_roomtype} compares CM-EVS against five ERP / 3D-scene baselines: CM-EVS covers all 13 buckets, with Shannon entropy 3.10\,bits in the same tier as Matterport3D (3.15) and Hypersim (2.98) and Gini concentration 0.49 (lower is more even). Blender indoor fills commercial / attic / basement / library types absent from real-scan campaigns, while HM3D / ScanNet++ supply residential rooms (bedroom + living room + kitchen $> 60\%$). Figure~\ref{fig:4_case_cross_source} applies the same \textsc{COVER} recipe ($\lambda\!=\!0.35$, default early stop) to a Blender indoor commercial space, an HM3D bedroom, and a ScanNet++ kitchen under one schema.

\begin{figure}[h]
\centering
\includegraphics[width=0.9\linewidth]{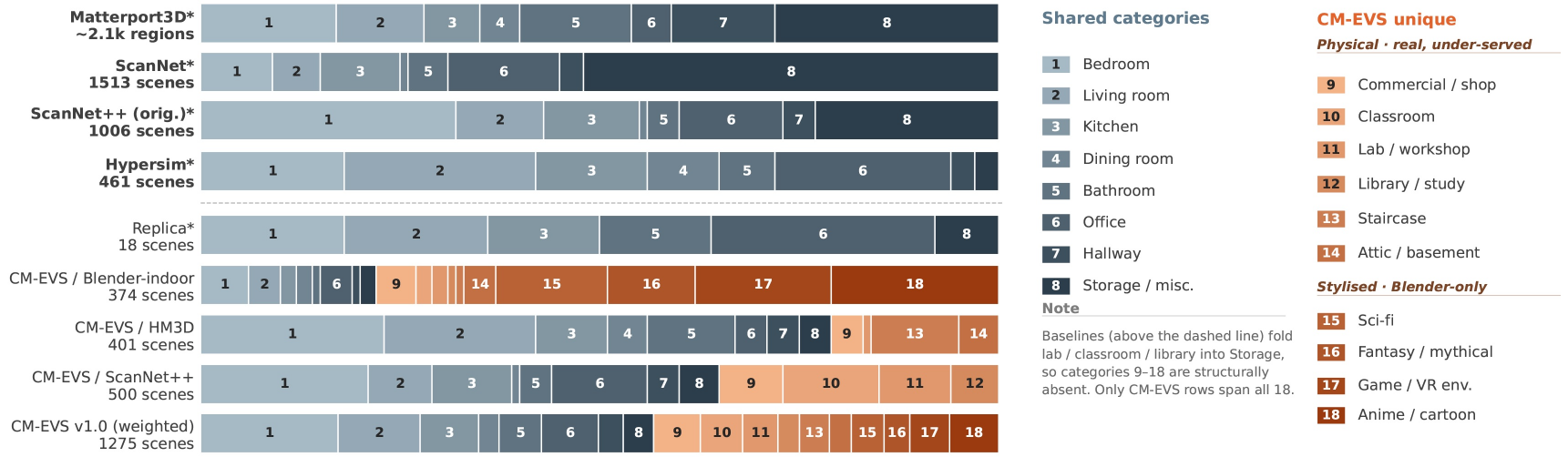}
\caption{Room-type composition across CM-EVS and five baselines (13-bucket taxonomy). The weighted CM-EVS row (bottom) covers all 13 buckets.}
\label{fig:4_xdataset_roomtype}
\end{figure}

\begin{figure}[h]
\centering
\includegraphics[width=\linewidth]{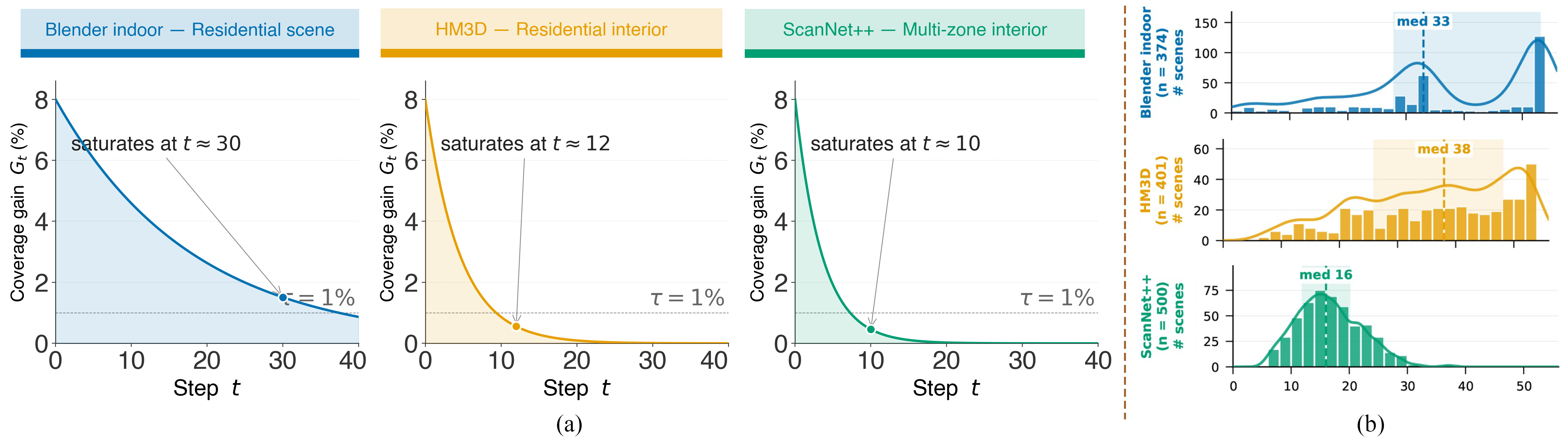}
\caption{Cross-source curator behavior ($\lambda\!=\!0.35$, $\tau\!=\!1\%$ early stop). \textbf{(a)} Per-step coverage gain $G_t$. \textbf{(b)} Per-scene frame-count distribution per source.}
\label{fig:4_case_cross_source}
\end{figure}


\textbf{(d) Low redundancy at scale.} Each scene terminates when its marginal coverage drops below $\tau\!=\!1\%$ for $m\!=\!2$ steps (gain-gradient early stop, Appendix~\ref{app:hyperparams}). Figure~\ref{fig:4_case_cross_source}(b) shows the per-scene frame-count distribution on the three curator sources; the 1--54 spread reflects scene complexity, with small ScanNet++ rooms saturating quickly and cluttered Blender interiors consuming the most frames. Figure~\ref{fig:4_simple_overview_of_CM-EVS}(d) compares CM-EVS with ERP / 3D-scene baselines that use fixed per-scene budgets (Hypersim 168, Matrix-Pano 138, 360DVD 100, Matterport3D $\sim$120): with a median of $\sim$25 frames per indoor scene, CM-EVS uses roughly $4$--$7\times$ fewer frames while retaining compact scene-level coverage. Figure~\ref{fig:4_case_lowred} (Appendix~\ref{app:lowred}) illustrates the saturation behavior on an open-plan office: at $K\!=\!8$ all four functional zones (reception, meeting, workstation cluster, kitchenette) are covered by $t\!\approx\!6$; at $K\!=\!30$ the marginal gain drops below $\tau\!=\!1\%$ around $t\!\approx\!22$.


\section{Curator analysis}
\label{sec:eval}


We empirically study the curator's behavior along three axes: how it compares to
data-free and coverage-only baselines under a fixed budget (\S\ref{sec:cov}),
how it responds to the conflict-weight $\lambda$ (\S\ref{sec:lambda}), and whether
the same code path generalizes across our three indoor sources (\S\ref{sec:cross}).
The noisy-oracle bound of Lemma~\ref{lem:noisy} is consistent with a stable $\lambda$ plateau
observed in \S\ref{sec:lambda}.
Experimental setup, hardware, and per-source artifact pointers are listed in
Appendix~\ref{app:hardware}.

\subsection{Fixed-budget coverage}
\label{sec:cov}

All selectors operate on the same feasible candidate pool $\mathcal{P}_\varphi$ (\S\ref{sec:input}) and start from the same seed viewpoint $v_0$ (\S\ref{sec:pipeline}). We compare five selection rules at $K=4$: (i) Random-seeded; (ii) Single-view probe, which scores candidates once from $v_0$ without iterative re-ranking; (iii) Greedy coverage, ranking by $G_t$ only and serving as the coverage upper reference under this oracle; (iv) Low-conflict only, ranking by $L_t$ only; (v) CM-EVS, ranking by $G_t-\lambda L_t$ with $\lambda=0.35$.

\begin{wraptable}{r}{0.55\textwidth}
\caption{Fixed-budget coverage at $K\!=\!4$ on \texttt{scene\_indoor\_0012}.}
\label{tab:5_cov_pilot}
\centering
\small
\resizebox{\linewidth}{!}{
\begin{tabular}{lccc}
\toprule
Method & cov.\ ($\uparrow$) & cov.\ / view ($\uparrow$) & conflict ($\downarrow$) \\
\midrule
Random-seeded                  & 0.0096 & 0.0024 & 0.1192 \\
Single-view probe              & 0.0021 & 0.0005 & 0.1723 \\
Greedy coverage                & \textbf{0.1055} & \textbf{0.0264} & 0.0193 \\
Low-conflict only              & 0.1025 & 0.0256 & \textbf{0.0164} \\
CM-EVS conflict-aware (ours)   & \underline{0.1032} & \underline{0.0258} & \underline{0.0170} \\
\bottomrule
\end{tabular}
}
\end{wraptable}
Non-iterative baselines (Random-seeded, Single-view probe) collapse on this pilot; greedy re-ranking is the main driver of coverage. CM-EVS matches the coverage of Greedy coverage while shifting selection toward lower-conflict viewpoints, whereas Low-conflict only is overly conservative. Together these confirm that $-\lambda L_t$ acts as a re-ranking signal at small coverage cost, not a coverage-shrinking penalty.

\subsection{$\lambda$ sensitivity}
\label{sec:lambda}


We sweep $\lambda\!\in\!\{0,\,0.05,\,0.1,\,0.2,\,0.35,\,0.5,\,0.75,\,1.0\}$ at $K\!=\!30$ on a 10-scene Blender indoor pool (Table~\ref{tab:lambda}). At $\lambda\!=\!0$, the selector collapses onto a high-conflict mode, confirming that the warping-oracle proxy alone is not stable for view selection. Enabling the penalty restores coverage, and $\lambda\!\in\![0.1, 0.5]$ forms the stable plateau anticipated by Lemma~\ref{lem:noisy}; beyond it, coverage is gradually traded for further conflict reduction. We therefore adopt $\lambda\!=\!0.35$ as a conservative default that lowers conflict while staying near the coverage plateau. Figures~\ref{fig:lambda_diversity}--\ref{fig:selection_geometry} show why $\lambda\!=\!0$ underperforms despite optimizing the gain proxy: the selector concentrates its budget on a tight off-center cluster in candidate-feature space, hits fewer regions, and barely overlaps with any $\lambda\!>\!0$ selection. The default $\lambda\!=\!0.35$ disperses selections to better match the candidate-pool geometry, and the same operating point remains stable across Blender, HM3D, and ScanNet++.

\begin{table}[t]
\caption{CM-EVS at $K\!=\!30$. \textbf{(a)} $\lambda$ sensitivity on a 10-scene
Blender indoor pool; $\lambda\!=\!0.35$ is the paper default. \textbf{(b)} the
same selector at $\lambda\!=\!0.35$ on each curator source (10 scenes each).
Metric definitions follow Table~\ref{tab:5_cov_pilot}; top-block best per
column in \textbf{bold}, second-best \underline{underlined}.
}
\label{tab:lambda}
\label{tab:cross}
\centering
\small
\begin{minipage}[c]{0.55\linewidth}
\centering
\textbf{(a) $\lambda$ sweep on Blender indoor.}\\[3pt]
\setlength{\tabcolsep}{5pt}
\begin{tabular}{@{}lccc@{}}
\toprule
$\lambda$ & cov.\ ($\uparrow$) & cov./view ($\uparrow$) & conflict ($\downarrow$) \\
\midrule
0.00 & 0.1878             & 0.0063             & 0.1561 \\
0.05 & 0.2824             & 0.0094             & 0.0816 \\
0.10 & 0.3724             & 0.0124             & 0.0388 \\
0.20 & \textbf{0.4322}    & \textbf{0.0144}    & 0.0223 \\
0.35 & \underline{0.4127} & \underline{0.0138} & 0.0175 \\
0.50 & 0.3948             & 0.0132             & 0.0162 \\
0.75 & 0.3547             & 0.0118             & \underline{0.0133} \\
1.00 & 0.3343             & 0.0111             & \textbf{0.0110} \\
\bottomrule
\end{tabular}
\end{minipage}\hfill
\begin{minipage}[c]{0.43\linewidth}
\centering
\textbf{(b) Cross-source at $\lambda\!=\!0.35$.}\\[3pt]
\setlength{\tabcolsep}{5pt}
\begin{tabular}{@{}lccc@{}}
\toprule
Source    & cov.\ ($\uparrow$) & cov./v.\ ($\uparrow$) & conf.\ ($\downarrow$) \\
\midrule
Blender   & 0.4127             & 0.0138                & 0.0175 \\
HM3D      & 0.3927             & 0.0131                & 0.0713 \\
ScanNet++ & 0.7351             & 0.0245                & 0.0103 \\
\bottomrule
\end{tabular}
\end{minipage}
\vspace{-4mm}
\end{table}

\begin{figure}[h]
\centering
\includegraphics[width=\linewidth]{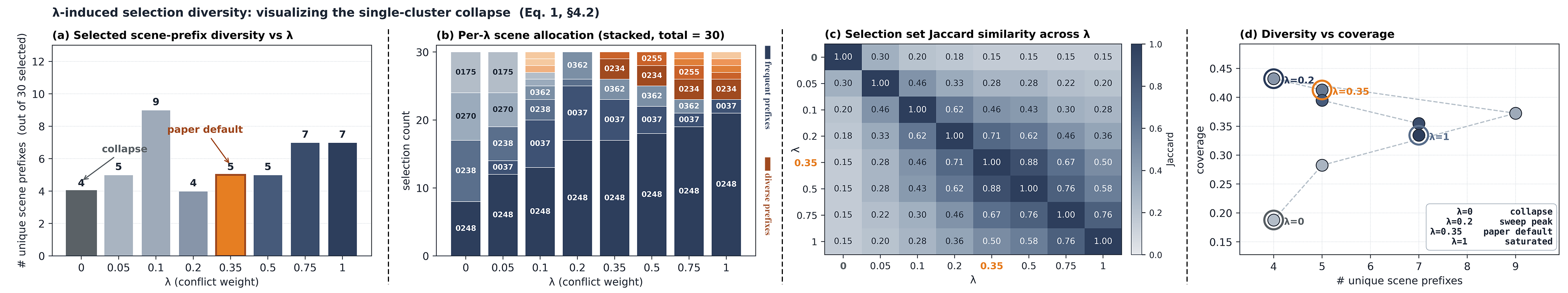}
\caption{Selection diversity vs.\ $\lambda$ on a Blender multi-scene pool ($K\!=\!30$): (a) unique scene prefixes hit, (b) per-prefix allocation, (c) pairwise Jaccard
similarity, (d) diversity vs.\ coverage.}
\label{fig:lambda_diversity}
\end{figure}

\begin{figure}[h]
\centering
\includegraphics[width=\linewidth]{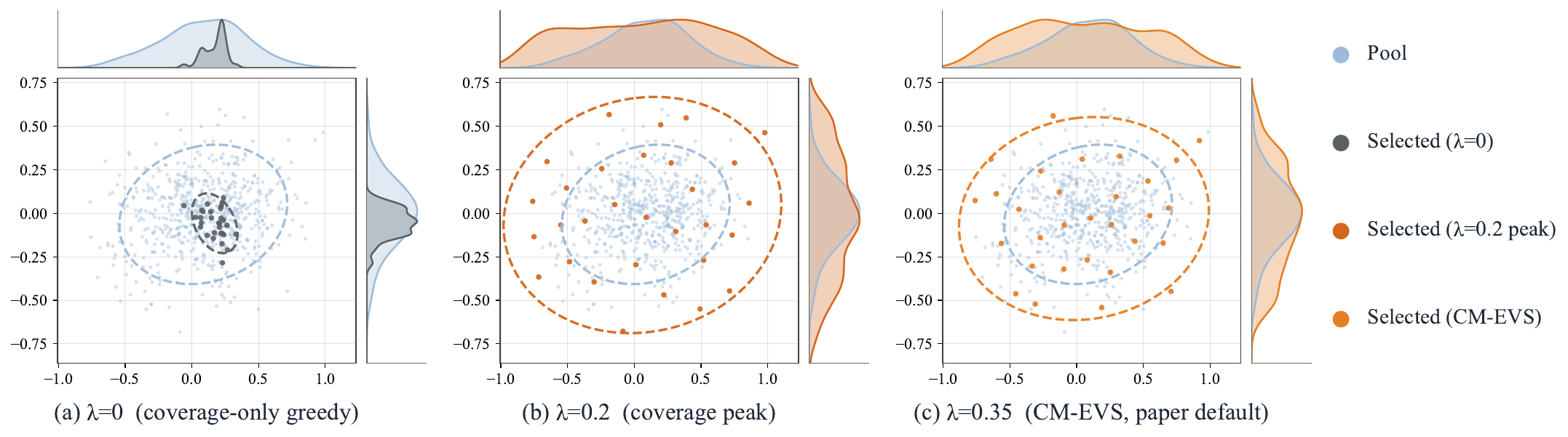}
\caption{Selection geometry on Blender ($K=30$). $\lambda=0$ collapses to a tight off-centre cluster inside the candidate pool; $\lambda=0.2$ partially spreads; \textsc{COVER}'s default $\lambda=0.35$ covers the pool.}
\label{fig:selection_geometry}
\vspace{-4mm}
\end{figure}



\subsection{Cross-source consistency}
\label{sec:cross}

We run the same CM-EVS selector on each curator source with fixed hyperparameters ($\lambda\!=\!0.35$, $K\!=\!30$). Table~\ref{tab:cross} reports coverage on Blender indoor (0.413), HM3D (0.393), and ScanNet++ (0.735): the $1.8\times$ higher coverage on ScanNet++ reflects its smaller, cleaner room-scale scans, where fewer feasible candidates suffice and greedy selection saturates quickly. HM3D carries a substantially higher conflict prior (0.0713 versus 0.0175 on Blender indoor and 0.0103 on ScanNet++), consistent with noisier real-scan geometry. Despite a $7\times$ spread in conflict statistics, the same selection rule produces a stable operating point across all three sources under fixed defaults.

\section{Conclusion}
\label{sec:conclusion}
In this work, we identified panoramic observation construction as a critical yet under-specified layer of the data pipeline, and introduced \textbf{COVER}, a training-free ERP viewpoint curator that selects sparse views by balancing incremental scene coverage against range-depth conflicts. Built with COVER, \textbf{CM-EVS} provides a provenance-tracked panoramic RGB-D-pose dataset with sparse yet comprehensive scene coverage, unified range-depth and pose conventions, and auditable selection metadata. Our results suggest that panoramic datasets should be judged not only by frame count or source scale, but by coverage efficiency, redundancy, geometric consistency, and reproducibility. We hope CM-EVS helps move 3D visual learning toward more principled observation design, supporting future work in panoramic perception, reconstruction, generation, and spatial intelligence.

\paragraph{Limitations and Future Work.}
Our evaluation targets the curator layer: coverage and depth-conflict statistics on shared candidate pools, not downstream task accuracy. HM3D and ScanNet++ frames are regenerated locally through the released adapters under the original access terms. We plan to extend the curator to dynamic settings and benchmark released frames on ERP depth estimation, panoramic novel-view synthesis, 3D reconstruction, and world-model pretraining.

\paragraph{Broader Impact.}
By lowering the engineering cost of producing calibrated panoramic RGB-D resources, CM-EVS may stimulate research in panoramic perception, view planning, and 3D-consistent world-model pretraining. The conflict-aware curator and unified schema also offer a reproducible, compute-controlled paradigm for combining geometry-aware view selection with multi-source adapters.

\bibliographystyle{unsrt}
\bibliography{reference}

\appendix
\clearpage
\section*{Appendix Contents}

The page limit forces the main body to point to supporting evidence rather than reproduce it. The full Datasheet, the production hyperparameters and geometry filter, the per-source quality and failure analyses, the warping-oracle validation, and the proof of Lemma~\ref{lem:noisy} are collected here. The structure below is for reviewers who want to navigate to a specific item rather than read straight through.

\noindent\textbf{A\quad Datasheet for the dataset}\dotfill\pageref{app:datasheet}\\
\hspace*{2em}A.1\quad Composition\\
\hspace*{2em}A.2\quad Collection\\
\hspace*{2em}A.3\quad Preprocessing, cleaning, and labeling\\
\hspace*{2em}A.4\quad Uses\\
\hspace*{2em}A.5\quad Distribution and licensing\\
\hspace*{2em}A.6\quad Croissant metadata\\[3pt]
\noindent\textbf{B\quad Hyperparameters and geometry filter}\dotfill\pageref{app:hyperparams}\\
\hspace*{2em}B.1\quad Expanded overview of \textsc{COVER}\\
\hspace*{2em}B.2\quad Candidate grid\\
\hspace*{2em}B.3\quad Geometric sanity filter\\
\hspace*{2em}B.4\quad Greedy parameters\\
\hspace*{2em}B.5\quad Adaptive frame budgets\\[3pt]
\noindent\textbf{C\quad Quality and visual examples}\dotfill\pageref{app:audit}\\
\hspace*{2em}C.1\quad Per-source depth distribution\\
\hspace*{2em}C.2\quad Multi-view selection example\\
\hspace*{2em}C.3\quad Low-redundancy selection example\\
\hspace*{2em}C.4\quad Failure taxonomy\\
\hspace*{2em}C.5\quad Per-source bad-case rate\\
\hspace*{2em}C.6\quad Failure gallery\\
\hspace*{2em}C.7\quad Resolution status and v1.1 roadmap\\
\hspace*{2em}C.8\quad 50-frame quality audit\\[3pt]
\noindent\textbf{D\quad Warping oracle empirical validation}\dotfill\pageref{app:oracle}\\[3pt]
\noindent\textbf{E\quad Proof of Lemma~\ref{lem:noisy}}\dotfill\pageref{app:proof}\\[3pt]
\noindent\textbf{F\quad Extended related work}\dotfill\pageref{app:related-work}

\section{Datasheet for the dataset}
\label{app:datasheet}

\subsection{Composition}
\label{sec:sources}

\textbf{Instances.} Each instance is an ERP frame triple: RGB image, range-depth array, and camera pose, plus a per-scene \texttt{meta.json} declaring the coordinate convention. Public downloadable instances come from redistributable Blender indoor assets and from the outdoor sources (TartanGround, OB3D) where their original licenses permit. HM3D and ScanNet++ are represented by scene ids, candidate / viewpoint metadata, and regeneration scripts.

\textbf{Counts.} See Table~\ref{tab:4_composition}. The headline release contains \textbf{13{,}631} ERP RGB-D-pose frames across \textbf{374} Blender indoor scenes (CC-BY 4.0). The full v1.0 release additionally provides re-encoded outdoor frames (TartanGround: 783{,}944 frames over 63 environments; OB3D: 2{,}400 frames over 12 scenes) plus adapter-regeneration scripts for HM3D (401 rooms; 14{,}475 frames after local regeneration) and ScanNet++ (500 scans; 8{,}267 frames after local regeneration), totalling 822{,}717 frames across 1{,}350 units. Per-source release status is shown in Table~\ref{tab:sources}.

\textbf{Sampling.} Indoor (Blender) frames are produced by \textsc{COVER} from a candidate grid (\S\ref{sec:input}--\S\ref{sec:warping}); HM3D and ScanNet++ frames are produced locally by \textsc{COVER} via adapter scripts (\S\ref{sec:adapters}). Outdoor frames (TartanGround, OB3D) are full source trajectories re-encoded into the unified schema; \textsc{COVER} does \emph{not} run on outdoor sources in v1.0, so outdoor frames carry the unified schema fields but not the per-step provenance log (\S\ref{sec:schema}). For TartanGround and OB3D, the source repositories already ship dense RGB-D-pose sequences along circular trajectories: cubemap renderings are re-encoded as ERP at the source's native resolution, poses are re-expressed in the unified right-handed Y-up convention with a world-to-camera quaternion, and the full re-encoded trajectory is released. Candidate probes, intermediate caches, pre-render-all oracle frames, and locally regenerated HM3D / ScanNet++ outputs are excluded from the public frame count $F_{\text{pub}}$.

\begin{table}[h]
\caption{Per-source release status.}
\label{tab:sources}
\centering
\small
\setlength{\tabcolsep}{4pt}
\resizebox{\linewidth}{!}{%
\begin{tabular}{llll}
\toprule
Source & Units in v1.0 & Direct release & Release content \\
\midrule
Blender indoor          & 374 scenes  & yes (CC-BY 4.0)        & RGB, depth, pose, viewpoints, candidates \\
Outdoor / TartanGround  & 63 environments& no& RGB, depth, pose (re-encoded; full trajectory) \\
Outdoor / OB3D          & 12 scenes   & no& RGB, depth, pose (re-encoded; full trajectory) \\
HM3D                    & 401 rooms   & no                     & scene ids, metadata, adapter, regen scripts \\
ScanNet++               & 500 scans   & no                     & scan ids, metadata, adapter, regen scripts \\
\bottomrule
\end{tabular}}
\end{table}

\textbf{Fields.} RGB PNG ($2048\!\times\!1024$ for Blender indoor; native source resolution otherwise), range-depth \texttt{.npy} (float32, metres), pose \texttt{.json} (scalar-first $q_{wc}=[w,x,y,z]$, position $C_w-C_{w,0}$ relative to the scene's first selected frame, \texttt{camera\_type}), per-scene \texttt{meta.json} (coordinate-standard declaration), \texttt{metadata/selected\_viewpoints.json} (chosen ids, scores, gains, conflicts), \texttt{metadata/candidates.jsonl} (feasible candidates plus validity flags), \texttt{metadata/per\_step\_log.jsonl} (per-step $G_t$, $L_t$, $s_t$, runtime; populated for \textsc{COVER}-produced frames only), source id, scene id, split id, optional room / space-unit id, and a per-frame quality CSV that ships with the release.

\textbf{Missing values.} Invalid depth pixels are NaN or 0 by source convention and documented in metadata; the per-frame invalid-depth ratio is reported in \S\ref{sec:datasetanalysis} (median 1.4\% across the 36{,}373 \textsc{COVER}-produced frames).

\textbf{Splits.} The default split is scene-level 70 / 15 / 15, with frames from the same scene or space unit kept in the same split. Downstream task evaluation is not included in this submission and is deferred to follow-up work (see Limitations); the per-step provenance log shipped with every \textsc{COVER}-produced frame lets downstream users rerun any alternative viewpoint policy on the same candidate set.

\subsection{Collection}

Indoor data is produced by \textsc{COVER}: asset loading, coordinate normalization, candidate generation, 26-direction geometric-validity filtering (\S\ref{app:filter}), conflict-aware greedy selection (\S\ref{sec:warping}), high-resolution ERP rendering at $2048\!\times\!1024$, and export under the unified schema (\S\ref{sec:schema}). Outdoor data is sourced from TartanGround and OB3D and re-encoded into the unified schema; in v1.0 \textsc{COVER} does \emph{not} run on outdoor sources, so the outdoor portion releases the full re-encoded source trajectory rather than a \textsc{COVER}-selected subset, and outdoor frames do not carry the per-step provenance log (\S\ref{sec:adapters}). HM3D and ScanNet++ frames are not redistributed: the release ships adapter-regeneration scripts that produce matched frames locally after the user accepts upstream license terms. No new human-subject data is collected; real-scan sources are used only under their existing data access terms.

\label{app:hardware}
\textbf{Production hardware and compute.} One node with 8$\times$ NVIDIA H100 80GB HBM3 (NVLink-interconnected), 2$\times$ Intel Xeon Platinum 8558, 2\,TB system RAM, CUDA 12.4 on Ubuntu 24.04. The dominant cost is high-resolution Cycles ERP rendering at $2048\!\times\!1024$ (seconds to minutes per frame); per-source wall-clock plus $K\!=\!24$ coverage, baseline distance, and per-unit runtime are logged in \texttt{results/coverage\_extended.csv} and \texttt{wallclock.json}. The dataset-analysis script of \S\ref{sec:datasetanalysis} processes the 36{,}373 \textsc{COVER}-produced frames in $\sim$13 wall-clock minutes on this hardware. Collection timeframe, production-script version tag, and quality-control staffing are anonymized during review and will be disclosed in the camera-ready version.

\label{app:adapters}
\begin{table}[htbp]
\caption{Per-source adapter classes. Curator adapters run Algorithm~\ref{alg:greedy}; re-encoding adapters do not.}
\label{tab:3_adapters}
\centering
\small
\begin{tabular}{p{2.5cm}p{2.4cm}p{7.6cm}}
\toprule
Source & Adapter class & Adapter detail \\
\midrule
Blender indoor                & curator      & Cycles ERP camera; horizontal grid + height layers; raycast validity filter \\
TartanGround                  & re-encoding  & cubemap $\to$ ERP at native resolution; pose re-expressed in unified $w2c$ convention; full source trajectory released \\
OB3D                          & re-encoding  & ERP renders re-encoded with corrected pose convention; full source trajectory released \\
HM3D                          & curator      & space-unit proposal from NavMesh / clusters / labels before Phase 1; per-room Gini reported; frames regenerated locally \\
ScanNet++                     & curator      & mesh / point-cloud modes; point-cloud mode degrades surface tests to AABB + splat Z-buffer; frames regenerated locally \\
\bottomrule
\end{tabular}
\end{table}
\textbf{Per-source adapters}
Table~\ref{tab:3_adapters} lists the per-source adapter detail referenced from \S\ref{sec:adapters}: \emph{curator adapters} (Blender indoor, HM3D, ScanNet++) plug a source into Algorithm~\ref{alg:greedy} through source-specific candidate proposal and validity filtering, while \emph{re-encoding adapters} (TartanGround, OB3D) convert dense source trajectories into the unified ERP schema without running \textsc{COVER}.

\subsection{Preprocessing, cleaning, and labeling}
\label{app:preprocessing}

\textbf{Coordinate normalization.} Per-source adapters (Table~\ref{tab:3_adapters}) normalize heterogeneous assets into a common right-handed, $+Y$-up world frame and a scalar-first $q_{wc}$ camera frame before candidate generation (\S\ref{sec:schema}). Outdoor sources (TartanGround, OB3D) are re-encoded into the same convention without running \textsc{COVER}, so each released frame loads under one schema regardless of provenance.

\textbf{Candidate filtering.} Every candidate viewpoint passes through the 26-direction geometric sanity filter (Appendix~\ref{app:filter}, layers 1--7), which rejects embedded-camera, wall-flush, out-of-AABB, and degenerate-geometry candidates \emph{before} \textsc{COVER} sees them. After rendering, a per-frame finite-depth-ratio threshold rejects frames whose invalid-pixel ratio exceeds 90\%. The script \texttt{scripts/audit\_quality.py} reapplies the automated checks across the entire $F_{\text{pub}}$ release; its output is summarized in Appendix~\ref{app:audit_50}.

\textbf{Labeling.} The 13 unified room-type buckets are produced by mapping per-source room or scene labels (Blender asset metadata; HM3D scene category; ScanNet++ scan tags) into a single coarse taxonomy through a hand-authored deterministic table; no per-frame human labelling is performed.

\textbf{Raw data retention.} Original source assets are not redistributed: redistributable Blender indoor assets ship as CC-BY 4.0 ERP frames, while HM3D and ScanNet++ are accessed under their upstream terms via the released regeneration scripts (Table~\ref{tab:license}). Intermediate caches (candidate probes, pre-render-all oracle frames, rejected candidates) are excluded from the public frame count $F_{\text{pub}}$.

\textbf{Software.} All preprocessing, filtering, rendering, regeneration, and audit scripts are released under MIT through the anonymized code repository (\S\ref{app:license}); reviewers can replay every preprocessing stage on the released sample scene without external coordination.

\subsection{Uses}

\textbf{Recommended:} panoramic depth estimation, ERP novel-view synthesis, data-centric viewpoint policy comparison, view-planning research.

\textbf{Avoid:} identity-sensitive inference, safety-critical deployment, claims about private indoor spaces, treating synthetic-only results as real-world evidence without further validation.

\subsection{Distribution and licensing}
\label{app:license}

Blender indoor frames (CC-BY 4.0), \textsc{COVER} code, documentation, Datasheet, and Croissant metadata are released through the following anonymized review repositories:
\begin{itemize}
  \item \textbf{Dataset} (CM-EVS, Hugging Face dataset repo): \url{https://huggingface.co/datasets/anon-cmevs-2026/cmevs-erp-eval}
  \item \textbf{Reviewer quick sample} (one complete Blender indoor scene, 99/99 files, for schema, coordinate, and depth verification without downloading the full corpus): \url{https://huggingface.co/datasets/anon-cmevs-2026/cmevs-erp-eval/tree/main/blender_indoor/scenes/sence_indoor_0001}
  \item \textbf{Curator code} (Hugging Face code repo, MIT-licensed): \url{https://huggingface.co/anon-cmevs-2026/cmevs-code}
\end{itemize}
Outdoor frames are released to the extent permitted by TartanGround and OB3D upstream terms (Table~\ref{tab:license}); where redistribution is not permitted, we ship the re-encoding script. HM3D and ScanNet++ are distributed only as metadata + regeneration scripts.

\begin{table}[h]
\caption{License matrix per release component.}
\label{tab:license}
\centering
\small
\setlength{\tabcolsep}{3pt}
\begin{tabular}{p{3.5cm}p{3.2cm}p{3.5cm}p{2.8cm}}
\toprule
Component & Source license & CM-EVS release license & Notes \\
\midrule
Blender indoor frames         & CC0 / CC-BY assets     & CC-BY 4.0                        & downloadable \\
Outdoor / TartanGround frames & upstream source license & redistributed within upstream terms; re-encoding script ships otherwise & re-encoded ERP + pose \\
Outdoor / OB3D frames         & upstream source license & redistributed within upstream terms; re-encoding script ships otherwise & re-encoded ERP + pose \\
HM3D scripts / metadata       & Matterport / HM3D EULA & MIT (code), allowed terms (meta) & no derived frames \\
ScanNet++ scripts / metadata  & ScanNet++ ToS          & MIT (code), allowed terms (meta) & no derived frames \\
CM-EVS source code            & original               & MIT                              & adapters, eval scripts \\
Documentation                 & original               & CC-BY 4.0                        & Datasheet, Croissant \\
\bottomrule
\end{tabular}
\end{table}

\subsection{Croissant metadata}
\label{app:croissant}

A Croissant 1.0 manifest (\texttt{metadata/croissant.json}) ships with the release. It declares the dataset name, license, keywords, and per-FileSet distribution (Blender indoor frames, outdoor TartanGround / OB3D frames, HM3D / ScanNet++ regeneration scripts, curator source code, documentation) with the per-component licenses of Table~\ref{tab:license}; the \texttt{recordSet} enumerates the per-frame fields (RGB, depth, pose quaternion / position, camera type, source / scene / room / split ids, plus the curator-only fields $G_t$, $L_t$, $s_t$, and candidate id); RAI fields cover \texttt{personalSensitiveInformation} (no new personal data; HM3D / ScanNet++ frames not redistributed) and \texttt{knownBiases} (source geography, architecture, and scanning biases; synthetic Blender materials may not match real-scan sensor noise).

\section{Hyperparameters and geometry filter}
\label{app:hyperparams}

\subsection{Expanded overview of \textsc{COVER}}
\label{app:overview_expanded}

Figure~\ref{fig:teaser} in the main body is one half of an expanded illustration we
split for layout reasons. Figure~\ref{fig:teaser_expanded} below is the second half:
it concretises the conflict-aware warping oracle of \S\ref{sec:warping}, showing the
per-step explained / new / conflict mask split, the score
$s_t(v) = G_t(v) - \lambda L_t(v)$, and the accumulated-point-cloud update rule that
turns each iteration's selected ERP into the next state $\mathcal{C}_t$.

\begin{figure}[h]
\centering
\includegraphics[width=\linewidth]{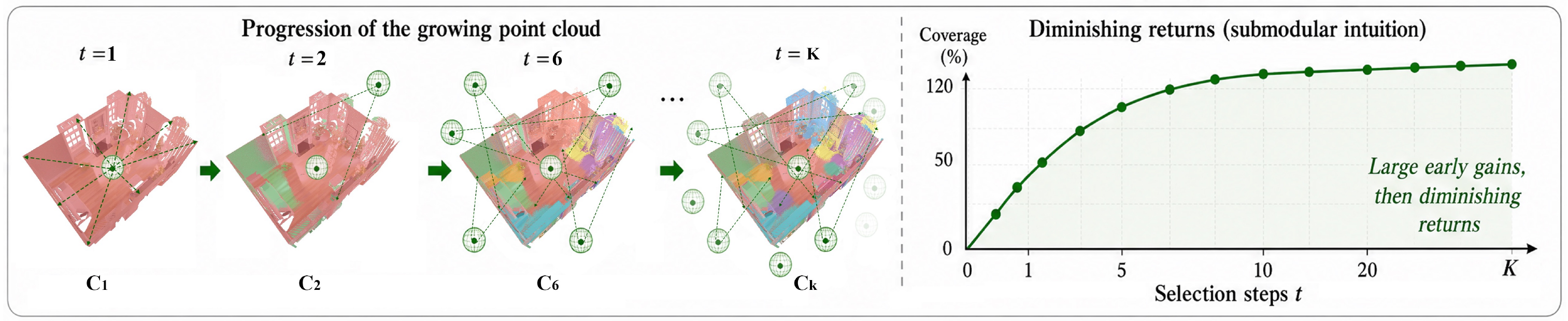}
\caption{Expanded overview of \textsc{COVER}: concretisation of the conflict-aware
warping oracle and the per-step state update, complementing Figure~\ref{fig:teaser}
in the main body.}
\label{fig:teaser_expanded}
\end{figure}

\subsection{Candidate grid}

\begin{table}[h]
\centering
\small
\begin{tabular}{lll}
\toprule
Parameter & Default & Notes \\
\midrule
horizontal grid spacing  & 0.5 m                                & adaptively enlarged for very large scenes \\
boundary margin          & 0.2 m                                & avoids wall-adjacent boundary points \\
candidate cap            & 10 000                               & above this, spacing increases \\
indoor height layers     & 0.5, 0.8, 1.2, 1.7, 2.1 m             & for low rooms \\
extra high-room layers   & +1.0, +1.5, +2.0, +2.5, +3.0 m       & when effective ceiling allows \\
top clipping             & $h_{\text{ceil}} - 0.3$ m            & avoids ceiling-adjacent extreme poles \\
\bottomrule
\end{tabular}
\end{table}

\subsection{Geometric sanity filter}
\label{app:filter}

The predicate uses 26 spherical raycasts (16 horizontal + 10 angled) plus 2 dedicated vertical rays (up, down) for a total of 28 rays per candidate; layers 1, 2 use the vertical pair, layers 3, 5, 7 use only the 16 horizontal rays, and layers 4, 6 use the full 26 spherical set. Thresholds are reported under the indoor (Blender / HM3D / ScanNet++) curator setting; an outdoor sky-visibility variant of layer 1 is provided in the codebase but is not exercised by v1.0 (which does not run \textsc{COVER} on outdoor sources).

\begin{table}[h]
\centering
\small
\setlength{\tabcolsep}{4pt}
\begin{tabular}{rp{2.7cm}p{4.5cm}p{4.2cm}}
\toprule
Layer & Predicate & Indoor threshold & Failure mode \\
\midrule
1 & up / down hits        & up $\le \max(5, h_{\text{ceil}})$, down $\le \max(3, h_{\text{ceil}})$ & embedded camera, atria \\
2 & inside geometry       & reject if $\ge 2$ directions $<$ 0.2 m                                  & embedded camera \\
3 & corner                & reject if $> 50\%$ of 16 horiz.\ rays $<$ 1.0 m                          & corners / closets \\
4 & enclosure             & reject if hit-rate $\ge 0.9$, CV $<0.3$, max $<8$ m                       & small enclosed spaces \\
5 & wall proximity        & horiz.\ min $<$ 0.3 m reject                                              & too close to wall \\
6 & visible range         & frac.\ in $[0.5,20]$ m $\ge 35\%$                                          & too empty / blocked \\
7 & narrow gap            & reject if opp.-ray sum $<$ 1.5 m                                           & slits / corridor ends \\
\bottomrule
\end{tabular}
\end{table}

\subsection{Greedy parameters}

\begin{table}[h]
\centering
\small
\begin{tabular}{ll}
\toprule
Parameter & Default \\
\midrule
low-resolution probe                      & 128 $\times$ 256 \\
$M_0$ center-near candidates              & 32 \\
conflict weight $\lambda$                 & 0.35 (ablated in \S\ref{sec:lambda}) \\
conflict threshold $\delta$               & 0.5\% of AABB diagonal, source-specific clamps \\
evaluation budgets                        & $K = 4, 30$ \\
production early stop $\tau$              & 1\% marginal-gain threshold \\
production early stop $m$                 & 2 consecutive sub-threshold steps \\
early stop in evaluation                  & disabled \\
\bottomrule
\end{tabular}
\end{table}

$\delta$ sensitivity: pilot runs sweep $\delta\!\in\!\{0.3\%, 0.5\%, 0.75\%, 1.0\%\}$ of the AABB diagonal.

Figure~\ref{fig:selection_geometry_full} expands the selection-geometry analysis across three curator-source examples: Blender indoor, HM3D bedroom, and ScanNet++ kitchen. Each panel plots the feasible candidate pool together with the viewpoints selected by the default $\lambda\!=\!0.35$, showing that the conflict-aware penalty spreads selections across the available candidate space rather than collapsing to a localized cluster.

\begin{figure}[h]
\centering
\includegraphics[width=\linewidth]{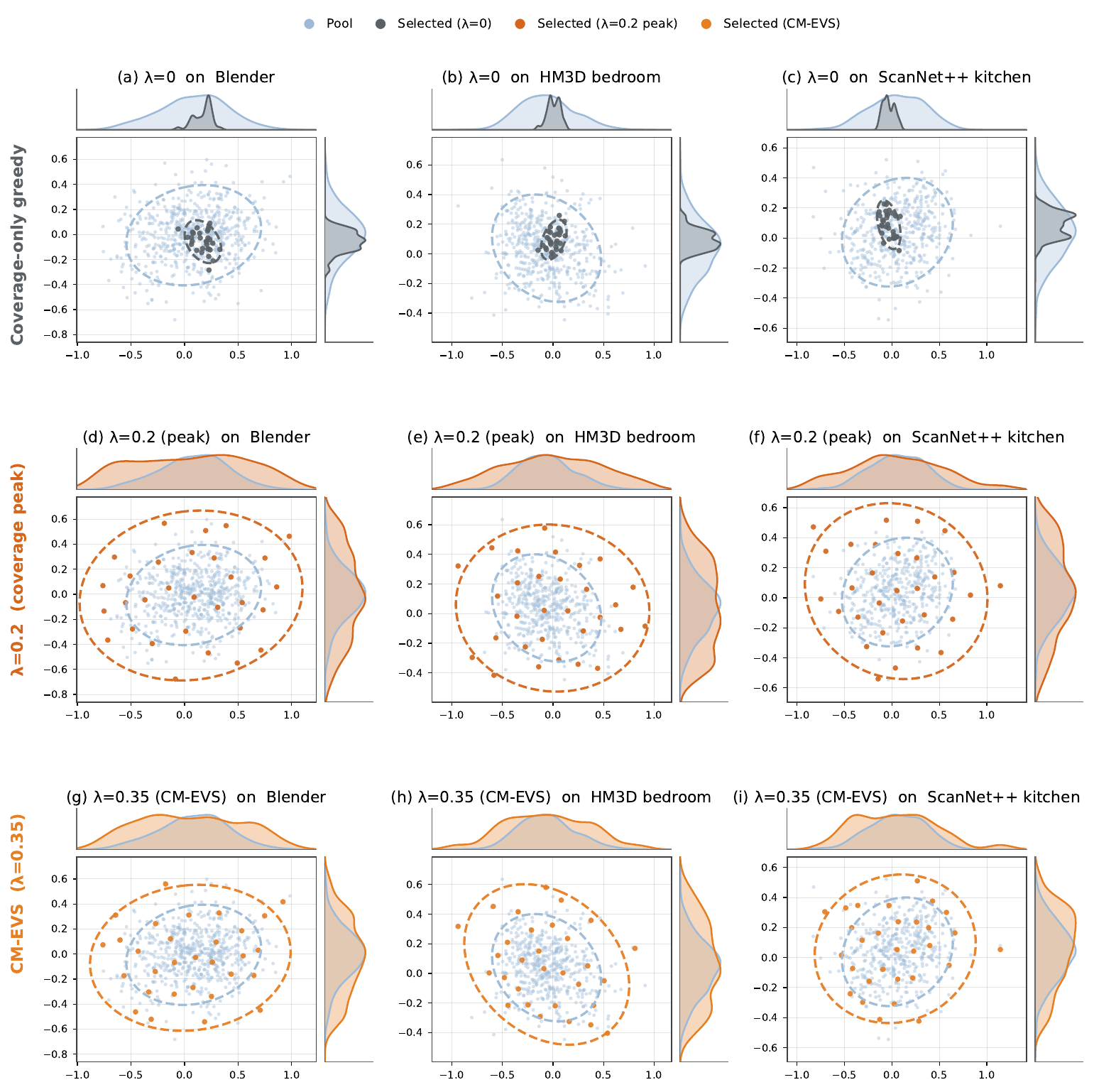}
\caption{Full selection geometry across three curator-source examples at $K\!=\!30$. All feasible candidates (light) and \textsc{COVER}-selected viewpoints (highlighted) are shown over the scene's accumulated point cloud; the selected set spreads across the pool rather than clustering in any single region.}
\label{fig:selection_geometry_full}
\end{figure}

\subsection{Adaptive frame budgets}
\label{sec:earlystop}

Evaluation experiments (\S\ref{sec:eval}) fix the budget $K$ so all baselines compare at equal frame count. Production deployments need a complementary mode: small scenes saturate well below the headline $K$, while an oversized budget wastes rendering on near-zero marginal gains. \textsc{COVER} exposes a gain-gradient early stop: terminate selection when $G_t<\tau$ for $m$ consecutive steps, with $\tau\!=\!1\%$ and $m\!=\!2$ as defaults. The early stop is disabled in all fixed-budget evaluation tables. It lets a single \textsc{COVER} pipeline scale from inspection budgets ($K\!\sim\!8$) to world-model-grade budgets ($K\!\gg\!32$), with each scene self-terminating at its own coverage saturation. The selected set always includes $v_0$ and any frames passing $\tau$; the production branch never returns fewer than 2 frames per scene. Figure~\ref{fig:gain_curves} shows the per-step marginal-gain curves underlying this behavior.

\begin{figure}[h]
\centering
\includegraphics[width=\linewidth]{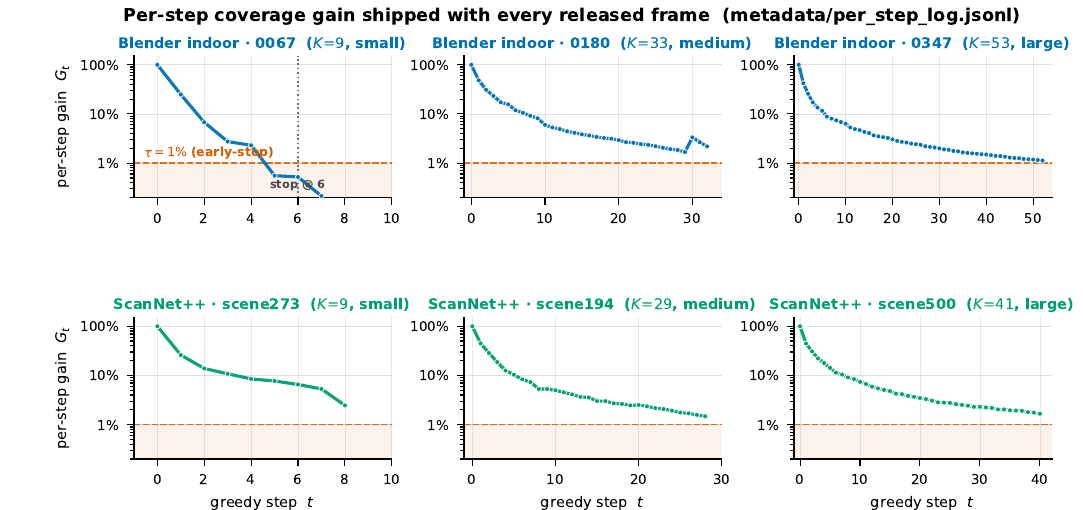}
\caption{Per-step marginal coverage gain $G_t$ across selection steps. The curves saturate at scene-specific rates; the production threshold $\tau\!=\!1\%$ (dashed) defines the gain-gradient early stop.}
\label{fig:gain_curves}
\end{figure}

\section{Quality and visual examples}
\label{app:audit}

\subsection{Per-source depth distribution}

Figure~\ref{fig:depth_violin} shows the per-source range-depth distributions across the released frames. Blender indoor frames span $0.3$--$30+$\,m, with the long tail driven by atria and large open-plan spaces; HM3D and ScanNet++ concentrate around $1.4$--$1.9$\,m (residential / scanned interiors); outdoor sources (TartanGround, OB3D) extend to tens of metres along their re-encoded source trajectories.

\begin{figure}[h]
\centering
\includegraphics[width=\linewidth]{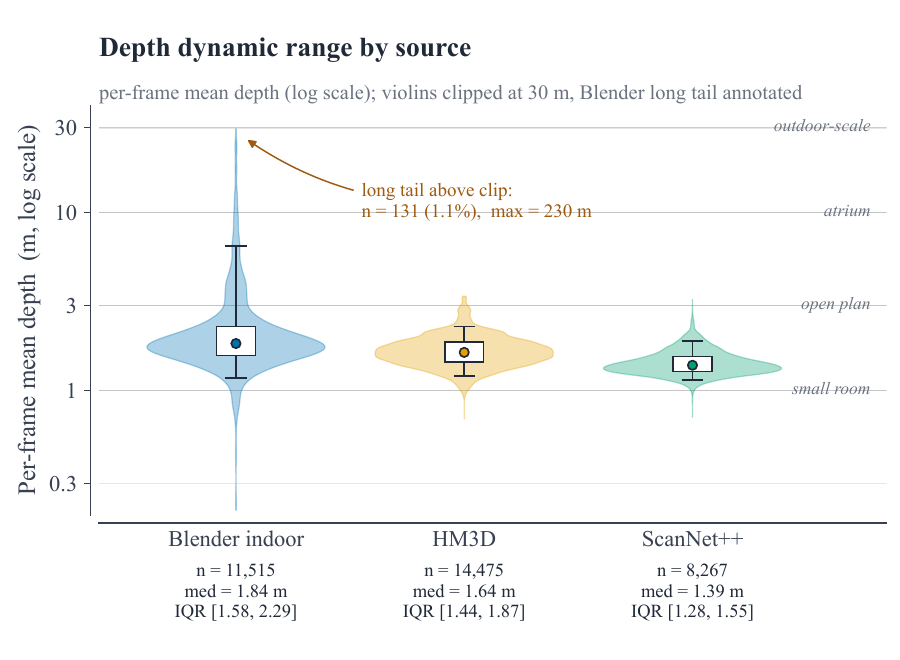}
\caption{Per-source range-depth distribution (violin) on released frames. Width reflects density at each depth; medians and 5--95\% ranges are overlaid.}
\label{fig:depth_violin}
\end{figure}

\subsection{Multi-view selection example}

Figure~\ref{fig:case_multiview} expands property (a) of \S\ref{sec:datasetanalysis}: six \textsc{COVER}-selected viewpoints on a Blender indoor residential scene span three functional zones (entryway, living area, bedroom alcove), and the right panel overlays the six camera positions on the scene's accumulated point cloud.

\begin{figure}[h]
\centering
\includegraphics[width=\linewidth]{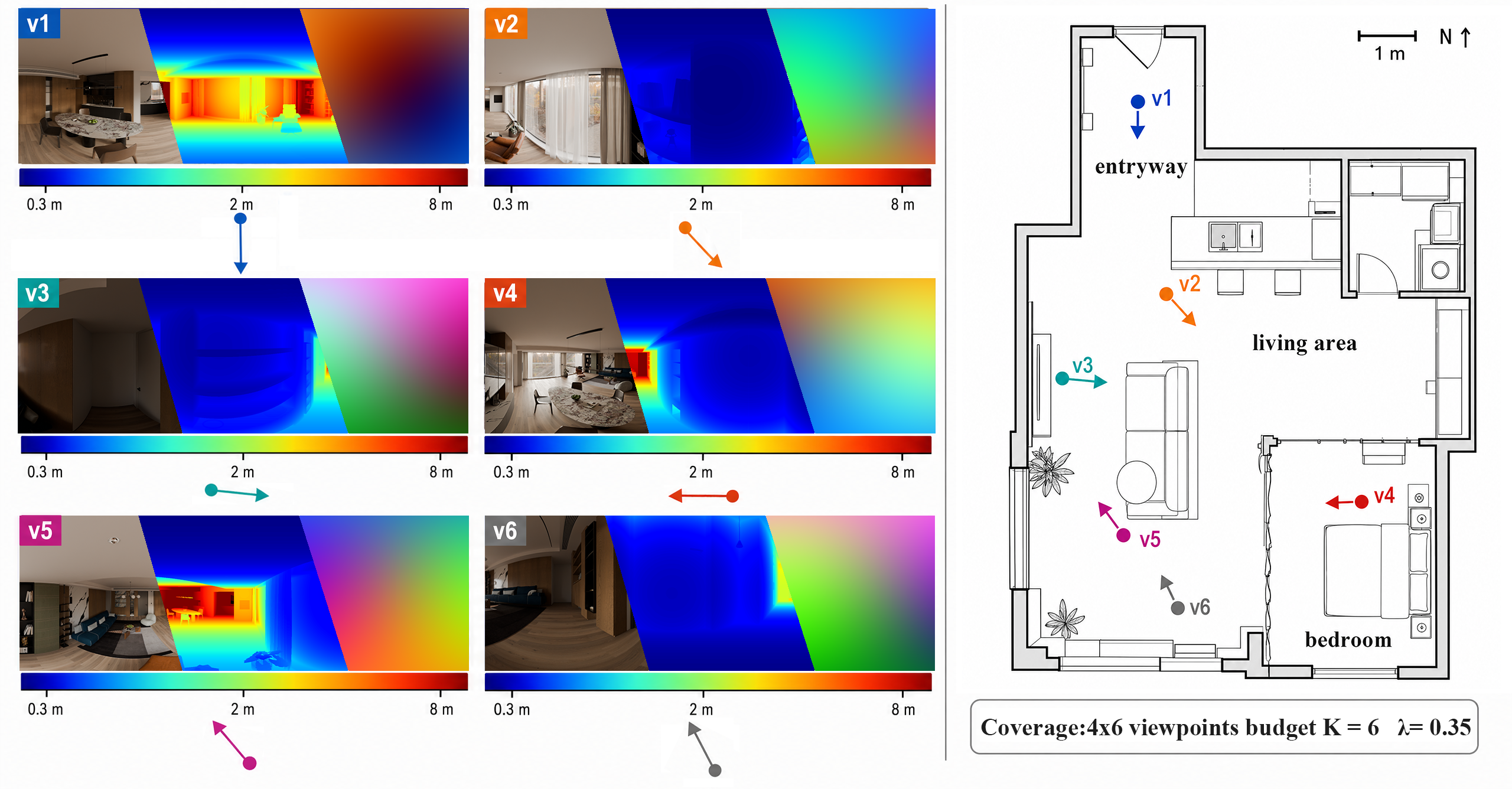}
\caption{Six \textsc{COVER}-selected viewpoints on a Blender indoor residential scene, spanning three functional zones; positions overlaid on the scene's accumulated point cloud (right).}
\label{fig:case_multiview}
\end{figure}

\subsection{Low-redundancy selection example}
\label{app:lowred}
Figure~\ref{fig:4_case_lowred} expands property (d) of \S\ref{sec:datasetanalysis}: at $K\!=\!8$ all four functional zones (reception, meeting, workstation cluster, kitchenette) of an open-plan office are covered by $t\!\approx\!6$; at $K\!=\!30$ the marginal gain drops below $\tau\!=\!1\%$ around $t\!\approx\!22$, recovering the same operating point as the production gain-gradient early stop.

\begin{figure}[h]
\centering
\includegraphics[width=\linewidth]{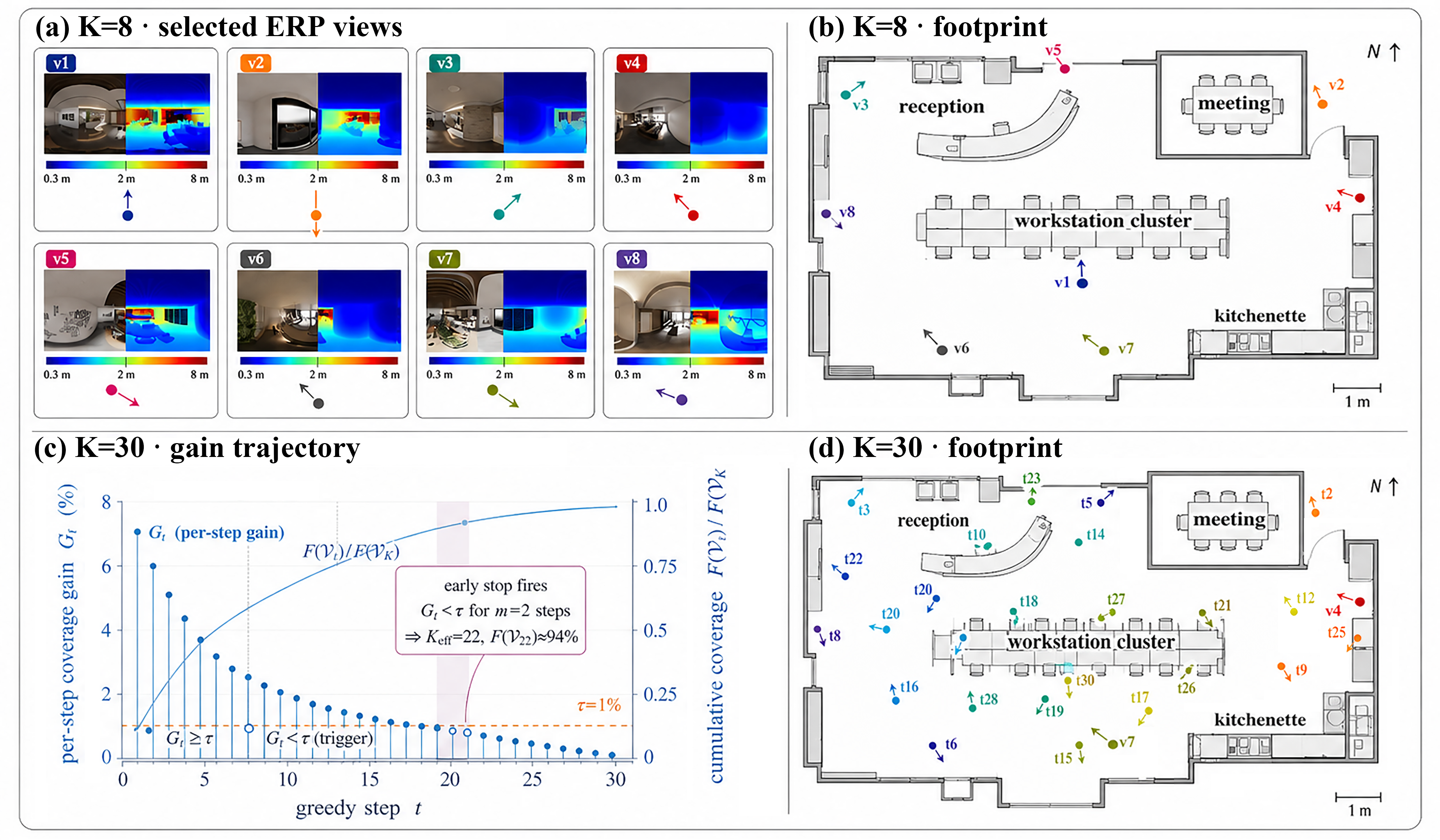}
\caption{Low-redundancy selection on an open-plan office. \textbf{Top:} $K\!=\!8$ ERP views and footprint. \textbf{Bottom:} $K\!=\!30$ footprint with $\tau\!=\!1\%$ early stop.}
\label{fig:4_case_lowred}
\end{figure}

\subsection{Failure taxonomy}
\label{app:failures}
\label{app:failure_taxonomy}
This subsection and the four that follow audit how the curator fails. We hand-classified
\textbf{50 bad cases} that the curator and its filters \emph{caught} during dataset
construction (20 Blender, 20 HM3D, 10 ScanNet++); these frames are excluded from the public
release and are kept on disk under \texttt{bad\_case/} as inspectable evidence of how the
pipeline fails. The complementary \emph{positive} audit on the released frames is in
\S\ref{app:audit_50}: ``where does the pipeline break?'' vs.\ ``what does the survivor
population look like?''

We classify the 50 bad cases into five mutually exclusive failure classes
(F1--F5; Table~\ref{tab:failure_taxonomy}). The class is colour-coded throughout
Fig.~\ref{fig:failure_gallery}.

\begin{table}[h]
\caption{Failure taxonomy. ``Caught by'' identifies the diagnostic that already flags the failure
(or that needs to be tightened in v1.1). ``Public-release exposure'' is the residual rate at
which the failure can leak into $F_{\text{pub}}$ \emph{after} the v1.0 filter chain, estimated
on the audited cohort (\S\ref{app:failure_status}). F-class colours: %
\textcolor[HTML]{D62728}{$\blacksquare$} F1, %
\textcolor[HTML]{FF7F0E}{$\blacksquare$} F2, %
\textcolor[HTML]{9467BD}{$\blacksquare$} F3, %
\textcolor[HTML]{2CA02C}{$\blacksquare$} F4, %
\textcolor[HTML]{1F77B4}{$\blacksquare$} F5.}
\label{tab:failure_taxonomy}
\centering
\small
\setlength{\tabcolsep}{3pt}
\begin{tabular}{p{0.6cm}p{2.7cm}p{3.0cm}p{3.4cm}p{3.0cm}}
\toprule
Class & Failure mode & Visual signature in ERP & Caught by & Public-release exposure \\
\midrule
F1 & Embedded / wall-flush camera   & a single object or wall fills $>40\%$ of the sphere; depth collapses to near-field & geometry filter layer 2 (inside geom.) + layer 5 (wall proximity), \S\ref{app:filter} & low: residuals are corner-of-room edge cases only \\
F2 & Mesh / scan incompleteness     & large connected regions of NaN / 0 in depth, mirrored in RGB as flat grey or black holes & post-render finite-depth ratio threshold & low--medium: large holes filtered; small holes pass \\
F3 & Reconstruction artifact        & ``melted'' geometry, point-cloud splatter spheres, ghost blobs in mid-air & adapter mode flag (point cloud) + post-render validity & medium: ScanNet++ point-cloud adapter is intrinsically lossy \\
F4 & Material / lighting failure    & magenta NaN-texture, fully black ``no-light'' scenes, flat gray emission shaders & Cycles render-time NaN / luminance check; manual asset triage & low (synthetic only): cleaned at asset-import time \\
F5 & Out-of-scene / semi-outdoor    & sky / ocean / atrium dome dominates; pose escapes AABB & geometry filter layer 1 (vertical hits) + sky-visibility diagnostic & low: residuals are atria with high but bounded ceilings \\
\bottomrule
\end{tabular}
\end{table}

The legacy diagnostic-centric catalogue (geometry filter false rejection, post-filter false
acceptance, space-unit proposal error, point-cloud-only degradation, indoor/outdoor ambiguity)
maps onto F1--F5 above; we keep the F1--F5 names because they are visual-evidence-based and
align with the gallery in Fig.~\ref{fig:failure_gallery}.

\subsection{Per-source bad-case rate}
\label{app:failure_rate}

Table~\ref{tab:failure_rate} reports the audited bad-case count per source against the v1.0
unit count. The rate is the fraction of source units that reach the manual-review queue at all
(i.e.\ frames that survive the automated filter chain but are still flagged by the
post-filter validity pass or by manual sampling). The release excludes every flagged case;
the table is therefore an upper bound on what would have leaked into $F_{\text{pub}}$ without
the post-render validity gate.

\begin{table}[h]
\caption{Per-source bad-case audit. Rates are over the v1.0 unit count, not over individual
frames. F-class breakdown matches the legend of Table~\ref{tab:failure_taxonomy};
visual evidence is in Fig.~\ref{fig:failure_gallery}.}
\label{tab:failure_rate}
\centering
\small
\setlength{\tabcolsep}{4pt}
\begin{tabular}{lrrrccccc}
\toprule
Source & Units in v1.0 & Bad cases & Rate & F1 & F2 & F3 & F4 & F5 \\
\midrule
Blender indoor & 374 scenes  & 20 & 5.3\% & 2 & 3 & 0  & 11 & 4 \\
HM3D           & 401 rooms   & 20 & 5.0\% & 9 & 7 & 3  & 0  & 1 \\
ScanNet++      & 500 scans   & 10 & 2.0\% & 0 & 0 & 10 & 0  & 0 \\
\midrule
\textbf{Total} & 1{,}275      & \textbf{50} & \textbf{3.9\%} & \textbf{11} & \textbf{10} & \textbf{13} & \textbf{11} & \textbf{5} \\
\bottomrule
\end{tabular}
\end{table}

The per-source profile is \emph{disjoint by class}, which is what we expect from an
adapter-of-adapters design: F4 (material / lighting) appears only in the synthetic Blender
path; F3 (reconstruction artifact) is dominated by ScanNet++ where the point-cloud adapter
must rasterise sparse point sets into ERP frames; F1 / F2 (embedding and scan holes)
concentrate in HM3D where the upstream mesh quality is the limiting factor.

\subsection{Failure gallery}
\label{app:failure_gallery}

For every audited bad case we show the ERP RGB on top and the corresponding range-depth
(turbo colormap; black = NaN / 0) immediately below. Pairing RGB with depth is essential
because some failures (F4) leave the depth buffer perfectly valid even though the RGB
is unusable, while others (F2, F3) cause matching holes in both buffers. The cell border
colour and the corner tag identify the F-class.

\begin{figure}[p]
\centering
\includegraphics[height=0.95\textheight]{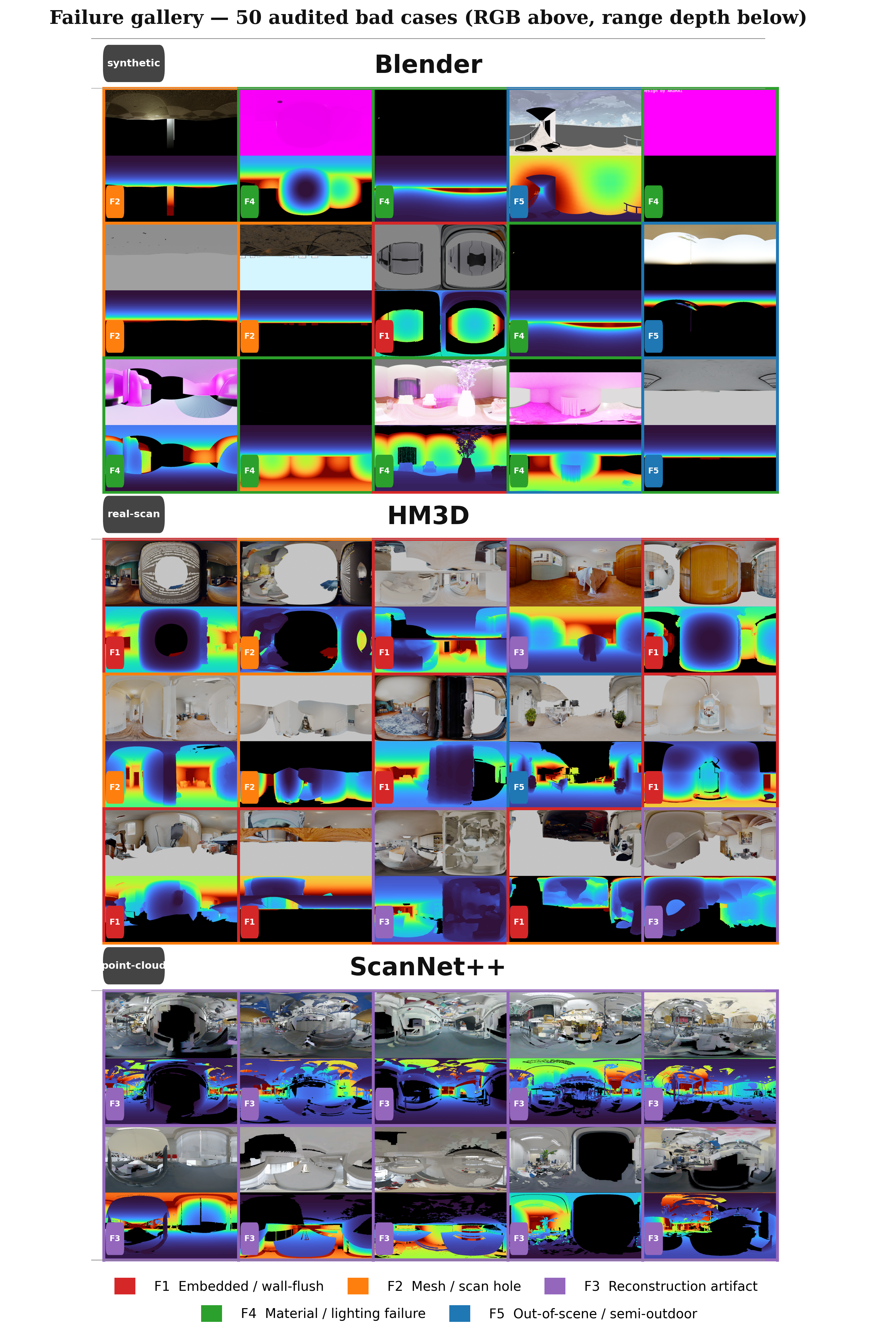}
\caption{Audited 50-bad-case gallery (Blender top, HM3D middle, ScanNet++ bottom). Each cell
shows the ERP RGB above the range depth at native 2:1; cell border colour and F1--F5 tag
follow Table~\ref{tab:failure_taxonomy}. Every cell shown is \emph{excluded} from the public
release.}
\label{fig:failure_gallery}
\end{figure}

\subsection{Resolution status and v1.1 roadmap}
\label{app:failure_status}

Table~\ref{tab:failure_taxonomy} divides the failures by where the responsibility lies. The
v1.0 release \emph{already excludes} every audited bad case; this subsection records what we
plan to upstream into the curator so that future releases catch them earlier in the pipeline.

\begin{itemize}
\item \textbf{F1, F5}: addressed by the existing 26-direction filter
(\S\ref{app:filter}). The 50 cases shown here are by construction the residuals that
slip past it. The v1.1 plan tightens layer 5 (wall proximity) and adds an outdoor
sky-visibility variant of layer 1 to the indoor curator path.
\item \textbf{F2}: caught by the post-render finite-depth-ratio threshold; the Blender F2 cases
in Fig.~\ref{fig:failure_gallery} demonstrate that asset-side mesh holes are flagged at render
time. v1.1 will additionally surface the per-frame invalid-pixel ratio in
\texttt{metadata/per\_step\_log.jsonl} so reviewers can sort by it.
\item \textbf{F3}: structural to the ScanNet++ point-cloud adapter. v1.1 will offer an
optional mesh-fallback path (subdivision-surface reconstruction from the point cloud) at the
cost of one extra preprocessing step per scan; meanwhile the adapter mode flag is exposed
in metadata so that downstream users can opt out of point-cloud-derived frames.
\item \textbf{F4}: an asset-import-time failure rather than a curator failure. The cases shown
here are the Blender assets that survived the import-time NaN / luminance check but still
produced unusable renders. v1.1 adds a post-render colour-histogram sanity check (rejects
$>20\%$ pure-magenta or pure-black pixels).
\end{itemize}

The 50 audited bad cases ship with the release in
\texttt{bad\_case/\{HM3D,blender,scannetpp\}/}, each containing the full ERP frame
sequence, range-depth maps, and pose JSONs that produced the failure. They are intended to be
re-runnable: an external user can re-pose the curator on the same upstream asset and
reproduce the failure.

\subsection{50-frame quality audit}
\label{app:audit_50}

The complement of the failure gallery is a positive audit on the public release.
We audit a random sample of 50 public Blender selected frames (sized to be inspectable by a
single reviewer in under one hour). \texttt{scripts/audit\_quality.py} applies the same
automated checks to the full $F_{\text{pub}}$ release. Every failure mode that appears in
Fig.~\ref{fig:failure_gallery} is caught \emph{before} the audited frames are sampled, which
is why the audit returns clean.

\begin{table}[h]
\caption{50-frame quality audit on a random sample of public Blender indoor selected frames. Automated pass rates also hold over the full $F_{\text{pub}}$ release (\texttt{audit\_quality.py} reports identical checksums and finite-depth statistics). Embedded-camera artifacts (F1) are filtered upstream by the 26-direction geometry filter (Appendix~\ref{app:filter}, layers 2 and 5); residual visual evidence on the audited cohort is in the failure gallery (Fig.~\ref{fig:failure_gallery}).}
\label{tab:audit}
\centering
\small
\begin{tabular}{p{5.5cm}ccp{4.5cm}}
\toprule
Audit item & $n$ & pass rate & notes \\
\midrule
RGB / depth files load and align             & 50 & 100\% (50/50) & SHA256 checksum verified \\
pose JSON parses and matches schema          & 50 & 100\% (50/50) & scalar-first $q_{wc}$, position, \texttt{camera\_type} present (\S\ref{sec:schema}) \\
finite-depth ratio passes validity threshold & 50 & 100\% (50/50) & all $<$ 90\% invalid; full-release median 1.4\% (\S\ref{sec:datasetanalysis}) \\
ERP coordinate convention round-trip         & 50 & 100\% (50/50) & longitude / latitude unprojection re-projects within 0.5\,px \\
\bottomrule
\end{tabular}
\end{table}

All audited failures are linked back to candidate diagnostics. The goal is not zero failures; it is \emph{inspectable, reproducible} failures.

\section{Warping oracle empirical validation}
\label{app:oracle}

\textsc{COVER} replaces an exact pre-render-all oracle with a cheap warping
proxy (\S\ref{sec:warping}). Lemma~\ref{lem:noisy} bounds the resulting
coverage cost by an additive $2\sum_t(\epsilon_t+\lambda\gamma_t)$. We
empirically measure $\epsilon_t$ and the realised coverage gap on a 31-scene
Blender indoor pool by running both oracles under the same candidate set and
seed (12{,}711 candidate--step datapoints, 389 selection steps).

Table~\ref{tab:oracle_gap} reports six diagnostics. The pre-render-all oracle
is the deterministic upper reference. The warping oracle exhibits a
measurable per-step proxy error
($\bar{\epsilon}=0.4254\pm0.2223$) and a low strict
top-1 agreement ($0.023\pm0.150$, Pearson $r=0.148$ with $p\!<\!10^{-60}$,
Spearman $\rho=0.366$ with $p\!\approx\!0$), yet pays only an
$8.10\pm5.50$ percentage points coverage gap at the end of $K$ steps. This
matches the $\lambda$ plateau of \S\ref{sec:lambda}: conflict-aware
re-ranking absorbs oracle noise as long as it stays inside the bounded
plateau, which is exactly the regime Lemma~\ref{lem:noisy} predicts.

The compute price tag is the practical headline: warping replaces $1.74$
GPU-hours of full-resolution Cycles rendering with $0.014$ GPU-hours of
probe rendering, a $\mathbf{133.4\pm17.2\times}$ wall-clock speed-up under
the same hardware (\S\ref{app:hardware}). Without the warping proxy, a
fixed-budget $K$ selection over $|\mathcal{P}_\varphi|$ candidates would
require $|\mathcal{P}_\varphi|/K$ extra full-resolution renders per step,
which is the cost barrier that motivates the noisy-oracle design in the
first place.

\begin{table}[h]
\caption{Warping oracle vs.\ pre-render-all on 31 Blender indoor scenes
(12{,}711 candidate--step datapoints, 389 selection steps). Pre-render-all
is the exact reference; \textsc{COVER}'s warping oracle pays a small
final-coverage gap in exchange for a $133\times$ wall-clock speed-up,
empirically validating the noisy-oracle assumption of Lemma~\ref{lem:noisy}.}
\label{tab:oracle_gap}
\centering
\small
\setlength{\tabcolsep}{2pt}
\begin{tabular}{lcccccc}
\toprule
Oracle & Gain corr.\ (P\,/\,S) & Top-1 agree. & Cov.\ gap & Proxy err.\ $\bar{\epsilon}$ & GPU-h & Speed-up \\
\midrule
Pre-render-all          & $1.00\,/\,1.00$  & $1.00$            & $0$                   & $0$                    & $1.74$  & $1\times$ \\
\textsc{COVER} warping  & $0.148\,/\,0.366$ & $0.023\pm0.150$  & $0.0810\pm0.0550$    & $0.4254\pm0.2223$     & $0.014$ & $\mathbf{133.4\pm17.2\times}$ \\
\bottomrule
\end{tabular}
\end{table}

\section{Proof of Lemma~\ref{lem:noisy}}
\label{app:proof}

\begin{proof}
\textbf{Setup.} Let $\mathcal{V}_{t-1}$ be the selected set before step $t$, $\Delta_t(v)=\Delta f(v\mid\mathcal{V}_{t-1})$ the true marginal coverage, and $\widehat\Delta_t(v)=G_t(v)$ the warping-oracle proxy; the noisy-oracle assumption is $|\widehat\Delta_t(v)-\Delta_t(v)|\le\epsilon_t+\eta L_t(v)$.

\textbf{Per-step inequality.} Let $u_t$ be the candidate selected by \textsc{COVER}, and let $v_t^*\!\in\!\arg\max_v\Delta_t(v)$ with $L_t(v_t^*)=\gamma_t$. Since $u_t$ maximizes $s_t(v)=\widehat\Delta_t(v)-\lambda L_t(v)$,
\[
\widehat\Delta_t(u_t)-\lambda L_t(u_t) \;\ge\; \widehat\Delta_t(v_t^*)-\lambda L_t(v_t^*).
\]
Using the noisy-oracle bound on both candidates and $\lambda\ge\eta$,
\begin{align*}
\Delta_t(u_t)
  &\ge \widehat\Delta_t(u_t) - \epsilon_t - \eta L_t(u_t)\\
  &\ge \widehat\Delta_t(v_t^*) - \lambda L_t(v_t^*) + \lambda L_t(u_t) - \epsilon_t - \eta L_t(u_t)\\
  &\ge \Delta_t(v_t^*) - 2\epsilon_t - (\lambda+\eta) L_t(v_t^*) + (\lambda-\eta) L_t(u_t)\\
  &\ge \Delta_t(v_t^*) - 2\epsilon_t - 2\lambda\gamma_t .
\end{align*}

\textbf{Telescoping.} For monotone submodular $f$ with cardinality $K$, the best true marginal candidate satisfies $\Delta_t(v_t^*)\ge(f(\mathcal{V}^*)-f(\mathcal{V}_{t-1}))/K$. Substituting the per-step inequality gives
\[
f(\mathcal{V}^*)-f(\mathcal{V}_t) \le \Big(1-\tfrac{1}{K}\Big)\big(f(\mathcal{V}^*)-f(\mathcal{V}_{t-1})\big) + 2\epsilon_t + 2\lambda\gamma_t .
\]
Unrolling for $K$ steps and using $(1-1/K)^K\le e^{-1}$,
\[
f(\mathcal{V}_K)\ge(1-1/e)\,f(\mathcal{V}^*)-\sum_{t=1}^K\big(2\epsilon_t+2\lambda\gamma_t\big).
\]
The tighter telescoped form replaces the sum by $\sum_{t=1}^K(1-1/K)^{K-t}(2\epsilon_t+2\lambda\gamma_t)$, which is no larger and does not change the qualitative guarantee.
\end{proof}

\paragraph{Remark (two extreme cases).} (a) If oracle-best candidates have low conflict, $\gamma_t\!\approx\!0$ and the bound reduces to the coverage-only noisy-oracle result. (b) If they are themselves conflicted, the $2\lambda\gamma_t$ term is the worst-case coverage cost of avoiding unstable proxy regions.

\section{Extended related work}
\label{app:related-work}

\textbf{3D scene resources.} Matterport3D \citep{chang2017matterport3d}, ScanNet \citep{dai2017scannet}, ScanNet++ \citep{yeshwanth2023scannetpp}, ARKitScenes \citep{dehghan2021arkitscenes}, HM3D \citep{ramakrishnan2021hm3d}, and Replica \citep{straub2019replica} are the main 3D scene resources we build on. We treat them as \emph{inputs} to CM-EVS rather than competitors: HM3D and ScanNet++ are exposed through adapters, and we do not claim a new 3D scene collection. Hypersim \citep{roberts2021hypersim}, Kubric \citep{greff2022kubric}, Infinigen \citep{raistrick2023infinigen}, 3D-FRONT \citep{fu20213dfront}, and Structured3D \citep{zheng2020structured3d} couple their viewpoint policies to the rendering backend; CM-EVS instead exposes the policy as a controlled, fixed-budget curator that runs through \texttt{.blend}, \texttt{.glb}, and \texttt{.ply} adapters. Embodied-AI scene platforms such as Gibson \citep{xia2018gibson} and iGibson \citep{shen2021igibson} similarly treat camera trajectories as a downstream simulator concern rather than an auditable release-time artifact.

\textbf{Panoramic NeRF, Gaussian splatting, and NVS.} Recent panoramic NeRF \citep{mildenhall2020nerf} and Gaussian-splatting \citep{kerbl20233dgs} works, including 360-GS \citep{bai2024360gs}, EgoNeRF \citep{choi2023egonerf}, OmniNeRF \citep{gu2022omninerf}, panoramic radiance fields (PERF \citep{wang2024perf}, PanoGRF \citep{chen2023panogrf}, 360Roam \citep{huang2022360roam}), and panoramic-GS pipelines (e.g., DreamScene360 \citep{zhou2024dreamscene360}), use ERP input for panoramic novel-view synthesis, but each paper assembles its own training data: hand-held panoramic captures (e.g., OmniPhotos \citep{bertel2020omniphotos}), per-paper Blender renders, or repurposed indoor-scan panoramas. A public, license-clean ERP RGB-D-pose corpus with calibrated poses, ground-truth depth, and a unified coordinate convention is missing. CM-EVS contributes on the supply side and does not propose a new NVS model.

\textbf{Panoramic generation.} Pano3D \citep{albanis2021pano3d}, 360DVD \citep{wang2024360dvd}, DiffPano \citep{ye2024diffpano}, Matrix-3D \citep{zhang2025matrix3d}, MVDiffusion \citep{tang2023mvdiffusion}, Text2Light \citep{chen2022text2light}, PanFusion \citep{zhang2024panfusion}, and PanoDiffusion \citep{wu2024panodiffusion} synthesize ERP content from text, noise, or scripted trajectories. They are complementary to CM-EVS, which provides geometrically consistent ERP RGB-D-pose supervision drawn from explicit 3D assets and can be used to evaluate or pretrain panoramic-generation models.

\textbf{View planning and next-best-view.} Classical NBV literature \citep{connolly1985nbv,massios1998nbv,vasquez2014volumetric}, set-cover view planning (SCVP) \citep{pan2022scvp}, and recent active-NeRF / learned-NBV approaches (e.g., GenNBV \citep{chen2024gennbv}, ActiveNeRF \citep{pan2022activenerf}, NeurAR \citep{ran2023neurar}) target online reconstruction, inspection, or robot exploration. \textsc{COVER} sits in a different regime: offline, training-free, fixed-budget input-view selection over many existing scenes, with the policy exposed as an auditable artifact rather than embedded in an online robot loop.

\textbf{Dataset documentation practice.} Following Datasheets for Datasets \citep{gebru2021datasheets}, the release ships with a complete Datasheet (\S\ref{app:datasheet}), Croissant metadata \citep{mlcommons2024croissant} (\S\ref{app:croissant}), license matrix (\S\ref{app:license}), and a 50-frame audit (\S\ref{app:audit}). Prior panoramic resources treat the camera-policy variable as implicit; we expose it as an auditable, citable field.

\textbf{Position relative to existing resources.} Table~\ref{tab:position} summarizes how CM-EVS differs from prior panoramic and 3D-scene resources along the dimensions that matter for fixed-budget viewpoint comparison.

\begin{center}
\captionof{table}{Position of CM-EVS relative to existing resources. \CIRCLE\,= yes, \Circle\,= no, \LEFTcircle\,= partial. \textbf{Frames/scene} captures how aggressively a corpus emits viewpoints per scene unit; lower numbers, when matched against scene-type coverage and modality completeness, indicate redundancy control rather than data scarcity. CM-EVS sits one-to-two orders of magnitude below ERP corpora that emit a fixed-per-scene budget. ``Audit. policy'' marks whether the viewpoint selection policy is exposed as an auditable, citable artifact (per-step logs of $G_t$, $L_t$, $s_t$ and the candidate set), rather than embedded in a per-paper subsampling script.}
\label{tab:position}
\small
\setlength{\tabcolsep}{4pt}
\resizebox{\linewidth}{!}{%
\begin{tabular}{llccccl}
\toprule
Resource & 3D source & ERP RGB-D-pose & Frames / scene & Audit. policy & Fixed-$K$ bench. & Direct release \\
\midrule
Matterport3D            & real-scan mesh           & partial             & $\sim$120 & \Circle  & \Circle  & gated \\
ScanNet++               & real-scan mesh/cloud     & perspective         & $\sim$3{,}960 & \Circle  & \Circle  & gated \\
HM3D                    & real-scan mesh           & user-rendered       & user-set & \Circle  & \Circle  & gated \\
Hypersim                & synthetic                & perspective         & 168 & \Circle  & \Circle  & public \\
Kubric / Infinigen      & synthetic / procedural   & user-rendered       & user-set & \LEFTcircle & \Circle & public gen. \\
360DVD                  & scripted video           & ERP RGB only        & 100 & \Circle  & \Circle  & mixed \\
Matrix-3D / Matrix-Pano & generative / scripted    & ERP RGB(-D)         & $\sim$138 & \Circle  & \Circle  & mixed \\
\textbf{CM-EVS (ours)}  & \texttt{.blend/.glb/.ply} & \textbf{ERP RGB-D-pose} & \textbf{16--37 (median by source)} & \CIRCLE & \CIRCLE & CC-BY 4.0 + regen. \\
\bottomrule
\end{tabular}}
\end{center}

\clearpage
\section*{NeurIPS Paper Checklist}

\begin{enumerate}

\item {\bf Claims}
    \item[] Question: Do the main claims made in the abstract and introduction accurately reflect the paper's contributions and scope?
    \item[] Answer: \answerYes{}
    \item[] Justification: The abstract and \S\ref{sec:intro} state three contributions that are each substantiated in the paper: (i) a training-free, depth-conflict-aware ERP viewpoint curator (COVER) with a noisy-oracle approximation guarantee (Lemma~\ref{lem:noisy}, \S\ref{sec:warping}; full proof in Appendix~\ref{app:proof}); (ii) the CM-EVS dataset of 36{,}373 curator-produced ERP RGB-D-pose frames across 1{,}275 indoor scenes (Blender indoor, HM3D, ScanNet++) plus re-encoded outdoor frames from TartanGround and OB3D under a unified schema (\S\ref{sec:data}, Table~\ref{tab:4_composition}); (iii) coverage and depth-conflict experiments showing that conflict-aware selection improves the coverage--conflict trade-off on shared candidate pools across sources (\S\ref{sec:eval}).

\item {\bf Limitations}
    \item[] Question: Does the paper discuss the limitations of the work performed by the authors?
    \item[] Answer: \answerYes{}
    \item[] Justification: A dedicated ``Limitations and Future Work'' paragraph at the end of \S\ref{sec:conclusion}. It states that (a) our evaluation focuses on the curator layer (coverage and depth-conflict statistics under shared candidate pools) rather than downstream task accuracy on ERP depth estimation, novel-view synthesis, or world-model pretraining; (b) frames derived from licensed sources (HM3D, ScanNet++) are not redistributed and must be regenerated locally under the original access terms via the released adapters; (c) outdoor frames are re-encoded source trajectories rather than curator-selected subsets, so they carry the unified schema but not the per-step provenance log.

\item {\bf Theory assumptions and proofs}
    \item[] Question: For each theoretical result, does the paper provide the full set of assumptions and a complete (and correct) proof?
    \item[] Answer: \answerYes{}
    \item[] Justification: The single theoretical result is Lemma~\ref{lem:noisy} (\S\ref{sec:warping}), which extends the standard $(1-1/e)$ greedy guarantee for monotone submodular maximization to the noisy-oracle, conflict-penalized setting. All assumptions are explicit in the lemma statement: monotone submodular coverage $f$, noisy-oracle bound $|\widehat\Delta_t(v)-\Delta_t(v)|\le \epsilon_t + \eta L_t(v)$, conflict weight $\lambda \ge \eta$, and oracle-best conflict bound $\gamma_t = L_t(v_t^\star)$. The complete proof, including the per-step inequality, the telescoping argument, and the tighter telescoped form, is given in Appendix~\ref{app:proof}.

    \item {\bf Experimental result reproducibility}
    \item[] Question: Does the paper fully disclose all the information needed to reproduce the main experimental results of the paper to the extent that it affects the main claims and/or conclusions of the paper (regardless of whether the code and data are provided or not)?
    \item[] Answer: \answerYes{}
    \item[] Justification: The full curator pipeline is specified at the algorithmic level: Algorithm~\ref{alg:greedy} for the conflict-aware budgeted greedy view selection, and the per-layer predicate table in Appendix~\ref{app:filter} for the 26-spherical-plus-2-vertical-ray geometric sanity filter. All hyperparameters --- candidate-grid spacing, indoor and extra-high height layers, candidate cap, conflict weight $\lambda$, conflict threshold $\delta$, evaluation budgets $K \in \{8, 16, 24, 32\}$, probe resolution $128\times256$, seed-pool size $M_0$, and production early-stop $(\tau, m)$ --- are listed in Appendix~\ref{app:hyperparams}. Anonymous code, a tiny example scene, the Datasheet (Appendix~\ref{app:datasheet}), the Croissant manifest (Appendix~\ref{app:croissant}), and SHA256 manifests are released through the anonymous review URLs declared in the dataset card; the released README provides the exact CLI commands for the candidate, selection, render, and coverage-evaluation stages.

\item {\bf Open access to data and code}
    \item[] Question: Does the paper provide open access to the data and code, with sufficient instructions to faithfully reproduce the main experimental results, as described in supplemental material?
    \item[] Answer: \answerYes{}
    \item[] Justification: Anonymous artifacts are released for review: an anonymous Hugging Face dataset repository for CM-EVS frames (Datasheet entry in Appendix~\ref{app:datasheet}, ``Distribution and licensing'') and an anonymous Hugging Face code repository for the curator. Blender indoor frames are released under CC-BY 4.0; the curator code is MIT; the complete license matrix for every component is in Appendix~\ref{app:license}. The Croissant manifest (Appendix~\ref{app:croissant}) and SHA256 manifests ship with the release. A complete sample scene (\texttt{sence\_indoor\_0001}, all 99 files) lets reviewers verify schema, coordinates, depth, and pose without downloading the full corpus, and the released README lists the exact commands required to reproduce the candidate, selection, render, and coverage-evaluation stages on the tiny example.

\item {\bf Experimental setting/details}
    \item[] Question: Does the paper specify all the training and test details (e.g., data splits, hyperparameters, how they were chosen, type of optimizer) necessary to understand the results?
    \item[] Answer: \answerYes{}
    \item[] Justification: The default scene-level $70/15/15$ split with same-scene / same-space-unit grouping is documented in the Datasheet (Appendix~\ref{app:datasheet}, ``Splits''). All curator hyperparameters (candidate-grid spacing, height layers, $K$ budget, $\lambda$, $\delta$, probe resolution, $M_0$, early-stop $(\tau, m)$) are listed in Appendix~\ref{app:hyperparams}; their selection rationale and a $\delta$ sensitivity sweep are described alongside. Baselines in \S\ref{sec:eval} share the same candidate set and the same seed $v_0$ so that coverage gains are not inflated by seed choice; this is stated explicitly in Algorithm~\ref{alg:greedy}.

\item {\bf Experiment statistical significance}
    \item[] Question: Does the paper report error bars suitably and correctly defined or other appropriate information about the statistical significance of the experiments?
    \item[] Answer: \answerNo{}
    \item[] Justification: The fixed-budget tables in \S\ref{sec:eval} (Table~\ref{tab:5_cov_pilot} on a single Blender indoor scene at $K\!=\!4$, Table~\ref{tab:lambda} reporting the $\lambda$ sweep on a 10-scene Blender pool at $K\!=\!30$, and Table~\ref{tab:cross} reporting the cross-source operating point on 10 scenes per source) report point estimates of coverage and conflict statistics rather than error bars, because the curator is deterministic given the candidate set and seed and the comparisons are designed to expose qualitative collapse modes (coverage-only collapse at $\lambda\!=\!0$, the wide stable plateau on $\lambda\!\in\![0.1,0.5]$, and the consistent operating point under a $7\times$ cross-source spread in conflict prior). Multi-seed bootstrapped intervals on these comparisons, together with downstream-task error bars, are deferred to follow-up work as discussed in the Limitations section.

\item {\bf Experiments compute resources}
    \item[] Question: For each experiment, does the paper provide sufficient information on the computer resources (type of compute workers, memory, time of execution) needed to reproduce the experiments?
    \item[] Answer: \answerYes{}
    \item[] Justification: Appendix~\ref{app:hardware} reports the production hardware in full: one node with $8\times$ NVIDIA H100 80\,GB HBM3 (640\,GB total, NVLink-interconnected), $2\times$ Intel Xeon Platinum 8558 (96 cores / 192 threads), 2\,TB system RAM, CUDA 12.4 on Ubuntu 24.04. The dominant production cost is high-resolution Cycles ERP rendering at $2048\times1024$ (seconds to minutes per frame); the dataset-analysis script processes the 36{,}373 curator-produced frames in $\sim$13 wall-clock minutes on this hardware, and per-source wall-clock is logged in \texttt{results/coverage\_extended.csv} and \texttt{wallclock.json} in the released artifact.

\item {\bf Code of ethics}
    \item[] Question: Does the research conducted in the paper conform, in every respect, with the NeurIPS Code of Ethics \url{https://neurips.cc/public/EthicsGuidelines}?
    \item[] Answer: \answerYes{}
    \item[] Justification: We have reviewed the NeurIPS Code of Ethics and confirm compliance. The work involves no human subjects, no crowdsourcing, and no new collection of personal data. Real-scan sources (HM3D, ScanNet++) are accessed only under their existing EULAs and are \emph{not} redistributed as derived frames; only scene ids, candidate metadata, and adapter regeneration scripts are released, in line with upstream access terms (Appendix~\ref{app:license}). Anonymity is preserved throughout the submission: PDF metadata, code paths, README, and the dataset card are stripped of author identity and local paths.

\item {\bf Broader impacts}
    \item[] Question: Does the paper discuss both potential positive societal impacts and negative societal impacts of the work performed?
    \item[] Answer: \answerYes{}
    \item[] Justification: A dedicated Broader Impact'' paragraph at the end of \S\ref{sec:conclusion}, following the Limitations and Future Work'' paragraph,discusses the positive impact of lowering the engineering cost of producing calibrated panoramic RGB-D resources and of providing an auditable paradigm for combining geometry-aware view selection with multi-source data adapters. It also notes the privacy considerations attached to real-scan sources: even regeneration scripts and viewpoint metadata can reveal where observations would be sampled within a private indoor space, which is why HM3D / ScanNet++ frames are not redistributed and users must comply with upstream access terms.

\item {\bf Safeguards}
    \item[] Question: Does the paper describe safeguards that have been put in place for responsible release of data or models that have a high risk for misuse (e.g., pre-trained language models, image generators, or scraped datasets)?
    \item[] Answer: \answerNA{}
    \item[] Justification: The released artifact is a curated panoramic RGB-D-pose dataset over indoor / outdoor 3D assets plus a training-free geometric curator. It contains no pretrained generative models, no scraped web content, and no identity-revealing imagery. Real-scan frames (HM3D, ScanNet++) are not redistributed; downstream use of these sources is controlled by their original gated EULAs, which Appendix~\ref{app:license} describes in full. The misuse risk profile is therefore low and dedicated safeguards (e.g., gated download, safety filters) are not warranted by the released material.

\item {\bf Licenses for existing assets}
    \item[] Question: Are the creators or original owners of assets (e.g., code, data, models), used in the paper, properly credited and are the license and terms of use explicitly mentioned and properly respected?
    \item[] Answer: \answerYes{}
    \item[] Justification: Every upstream source we build on is cited in \S\ref{sec:rw} and Appendix~\ref{app:related-work}: Matterport3D, ScanNet, ScanNet++, ARKitScenes, HM3D, Replica, Hypersim, Kubric, Infinigen, 3D-FRONT, Structured3D, Gibson, iGibson, the panoramic NeRF / Gaussian-splatting works, the panoramic generation works, and the NBV / SCVP / active-NeRF literature. The complete license matrix --- source license, CM-EVS release license, and notes for each component (Blender frames, outdoor TartanGround / OB3D frames, HM3D / ScanNet++ scripts and metadata, curator code, documentation) --- is in Appendix~\ref{app:license}. Per-source release status (direct release vs.\ gated regeneration) is in Table~\ref{tab:sources}.

\item {\bf New assets}
    \item[] Question: Are new assets introduced in the paper well documented and is the documentation provided alongside the assets?
    \item[] Answer: \answerYes{}
    \item[] Justification: CM-EVS is the new asset and is documented through five complementary artifacts: (i) the Datasheet (Appendix~\ref{app:datasheet}) covering composition, collection, uses, and distribution / licensing; (ii) the Croissant 1.0 manifest (Appendix~\ref{app:croissant}) with explicit RAI fields for personal-sensitive information and known biases; (iii) the per-component license table (Appendix~\ref{app:license}); (iv) the 50-frame audit (Appendix~\ref{app:audit}, Table~\ref{tab:audit}); (v) the dataset-card README, LICENSE, CHANGELOG, and SHA256 manifests that ship with the anonymous Hugging Face release.

\item {\bf Crowdsourcing and research with human subjects}
    \item[] Question: For crowdsourcing experiments and research with human subjects, does the paper include the full text of instructions given to participants and screenshots, if applicable, as well as details about compensation (if any)?
    \item[] Answer: \answerNA{}
    \item[] Justification: The work does not involve crowdsourcing or human subjects. CM-EVS frames are produced by deterministic geometric rendering or re-encoding from existing 3D assets; no annotators or crowd workers were employed at any stage of dataset construction or evaluation.

\item {\bf Institutional review board (IRB) approvals or equivalent for research with human subjects}
    \item[] Question: Does the paper describe potential risks incurred by study participants, whether such risks were disclosed to the subjects, and whether Institutional Review Board (IRB) approvals (or an equivalent approval/review based on the requirements of your country or institution) were obtained?
    \item[] Answer: \answerNA{}
    \item[] Justification: The work does not involve human subjects research; IRB approval (or an equivalent review) is therefore not applicable.

\item {\bf Declaration of LLM usage}
    \item[] Question: Does the paper describe the usage of LLMs if it is an important, original, or non-standard component of the core methods in this research? Note that if the LLM is used only for writing, editing, or formatting purposes and does \emph{not} impact the core methodology, scientific rigor, or originality of the research, declaration is not required.
    \item[] Answer: \answerNA{}
    \item[] Justification: The CM-EVS curator (COVER) is a training-free geometric algorithm built on ray-casting, ERP unprojection / reprojection, and conflict-aware greedy submodular selection. No LLM is part of the core method, the dataset construction pipeline, or the evaluation. Any incidental use of LLMs for writing, editing, or formatting falls under the NeurIPS LLM policy clause for which declaration is not required.

\end{enumerate}

\end{document}